\PassOptionsToPackage{table,xcdraw}{xcolor}
\documentclass[nonacm,manuscript,screen]{acmart}
\usepackage{array}
\usepackage{textcomp}
\usepackage{stfloats}
\usepackage{url}
\usepackage{verbatim}
\usepackage{caption}
\usepackage{subcaption}
\usepackage{comment}
\usepackage{multicol}
\usepackage{siunitx}
\usepackage{physics} 
\usepackage{multirow}
\usepackage{enumitem}

\usepackage[ruled, lined, linesnumbered, commentsnumbered, longend, noend]{algorithm2e}

\newcounter{problem}
\newenvironment{problem}[1][Problem]
  {\refstepcounter{problem}\begin{trivlist}\item[\hskip \labelsep {\bfseries #1 \theproblem.}]}
  {\end{trivlist}}

\theoremstyle{definition}
\newtheorem{definition}{Definition}[section]
\newtheorem{theorem}{Theorem}[section]
\newtheorem{lemma}[theorem]{Lemma}

\AtBeginDocument{%
  }

\begin{document}

\title{Stable Trajectory Clustering: An Efficient Split and Merge Algorithm}

\author{Atieh Rahmani}
\affiliation{%
  \institution{Department of Computer Science and Information Technology, Institute for Advanced Studies in Basic Sciences (IASBS)}
    \city{Zanjan}
  \country{Iran}}
\email{atieh.rhmn@gmail.com}
\author{Mansoor Davoodi}
\affiliation{%
  \institution{Department of Computer Science and Information Technology, Institute for Advanced Studies in Basic Sciences (IASBS)}
  \city{Zanjan}
  \country{Iran}
}
\affiliation{%
  \institution{Center for Advanced Systems Understanding (CASUS), Helmholtz-Zentrum Dresden-Rossendorf}
  \city{Görlitz}
  \country{Germany}
}
\email{mdmonfared@iasbs.ac.ir}

\author{Justin M. Calabrese}
\affiliation{%
  \institution{Center for Advanced Systems Understanding (CASUS), Helmholtz-Zentrum Dresden-Rossendorf}
  \city{Görlitz}
  \country{Germany}
}
\affiliation{%
  \institution{Helmholtz Centre for Environmental Research – UFZ}
  \department{Department of Ecological Modelling}
  \city{Leipzig}
  \country{Germany}
}
\affiliation{%
  \institution{University of Maryland}
  \department{Department of Biology}
  \city{College Park}
  \state{MD}
  \country{USA}
}
\email{j.calabrese@hzdr.de}

\begin{abstract}
  
\noindent Clustering algorithms fundamentally group data points by characteristics to identify patterns. Over the past two decades, researchers have extended these methods to analyze trajectories of humans, animals, and vehicles, studying their behavior and movement across applications. 
\noindent This paper presents whole-trajectory clustering and sub-trajectory clustering algorithms based on DBSCAN line segment clustering, which encompasses two key events: split and merge of line segments. The events are utilized to capture object movement history based on the average Euclidean distance between line segments.
 In this framework, whole-trajectory clustering considers entire entities' trajectories, whereas sub-trajectory clustering employs a sliding window model to identify local similarity patterns. 
Many existing trajectory clustering algorithms respond to temporary anomalies in data by splitting trajectories, which often obscures otherwise consistent clustering patterns and leads to less reliable insights. To address this, we introduce the \textit{stable} trajectory clustering algorithm, which leverages the mean absolute deviation concept to demonstrate that selective omission of transient deviations not only preserves the integrity of clusters but also improves their stability and interpretability.
We evaluate all proposed algorithms on real trajectory datasets to illustrate their effectiveness and sensitivity to parameter variations. 
\end{abstract}




\maketitle

\section{Introduction}
Trajectory clustering is one of the most important data analysis methods developed for analyzing moving objects like humans, animals, and vehicles in various fields and applications. It can be used to identify common behaviors and recognize similar movement patterns, as well as to detect irregular or anomalous movements in different types of trajectory data. Among the use cases of trajectory clustering, we can mention user lifestyle profiling to capture behavioral similarities beyond geographic location~\cite{63}, exploring urban travel patterns that reflect the regularity of mobility patterns, which is an important factor in evaluating the optimality of city spatial structure and road network connectivity~\cite{7},  classifying animal behaviors (foraging, resting, walking) using high-frequency tracking data to study their environmental interactions~\cite{6} and clustering animals by species based on their daily movement patterns, which is complicated by the erratic nature of these movements~\cite{2}. Trajectory clustering is also used in air traffic categorization~\cite{3,5}, traffic network identification~\cite{4}, and weather prediction~\cite{8}.
\par
Anomalies frequently occur in trajectory clustering due to irregularities or deviations in movement patterns. To address this, numerous algorithms have been developed, each designed to effectively detect and analyze these anomalies~\cite{20,46}. 
Missing, however, is an algorithm that can ignore insignificant anomalies in trajectories to reveal clustering patterns that are otherwise stable.
Current methods usually handle these deviations in two ways, but both have limitations. The first strategy is to preprocess the raw trajectory data to remove small deviations. This is not ideal because it might delete important information and change the true nature of the trajectory. A more common strategy is to split a trajectory from its cluster as soon as a deviation appears. It often creates fragmented clusters, which hide the main movement pattern and make the results hard to understand.
In recent years, researchers have focused on stable algorithms for clustering data points that select the optimal number of clusters~\cite{18}. In other words, current stable clustering algorithms try to split and merge data points correctly in a way that is robust against random oscillations.
Inspired by advances in stable clustering of data points, we propose a stable algorithm for trajectory clustering that can be applied on existing trajectory clustering methods. In this paper, we apply the stable algorithm to a trajectory clustering approach based on line segment clustering, which naturally produces split and merge events over time. Earlier studies also employed line segment clustering, but without incorporating temporal labels or explicitly defining split and merge operations in trajectory clustering \cite{16,17}. The core contribution of our paper is the stable trajectory clustering algorithm itself, which efficiently improves cluster quality by determining when outlier trajectories can be assigned to their respective clusters despite insignificant anomalies. It addresses a prominent issue in real trajectory data that previous algorithms have not addressed. 
\par
Building upon existing trajectory clustering algorithms, we propose algorithms for both whole-trajectory clustering and sub-trajectory clustering. A trajectory can be described as a piecewise linear function, where each segment represents the movement of an object over a time interval. Consequently, for trajectory clustering, a set of line segments exists for each time interval, and these sets need to be clustered. For this purpose, Density-Based Spatial Clustering of Applications with Noise (DBSCAN) line segment clustering defines the split and merge procedures of line segments. One key advantage of this approach is that the results from clustering in each time interval can be reused in the subsequent interval, as line segments often retain their previous cluster assignment.
The clustering process relies on the continuous average Euclidean distance as the similarity measure between line segments. Therefore, a data analyst can identify the behaviors of moving objects across all time intervals and determine which objects move together and belong to the same cluster in each interval. Trajectories whose line segments belong to the same cluster and show similar behaviors in each time interval are eventually grouped into a single whole-trajectory cluster. In addition to whole-trajectory clustering, the proposed framework includes a sub-trajectory clustering algorithm that focuses on specific time intervals to identify similar sub-trajectories. It employs a sliding window model with two parameters to capture temporal patterns and discover recurrent sub-sequences.
We then present an innovative definition and algorithm for the stable trajectory clustering algorithm, which is a first-of-its-kind advancement. 
The algorithm ignores minor and insignificant deviations of trajectories (outlier trajectories) that temporarily leave their clusters and then return to them. The approach reflects the true movement patterns more accurately, which makes the clusters more meaningful and easier to interpret.
Comparable situations occur in real-world movement patterns where short departures do not change the general movement pattern.
For instance, in urban traffic analysis, a vehicle making a brief stop at a traffic light or a pedestrian crossing should still be classified within the main traffic cluster. Furthermore, a person briefly deviating from a common trajectory to look at a shop window should not be marked as an outlier in pedestrian movement studies. A related example can be seen in animal movement studies, where a migrating bird or herd of animals temporarily deviates from its migration trajectory to find food or rest, but then returns to their original trajectory afterward.
In this algorithm, we employ the mean absolute deviation concept to assess how much each outlier trajectory deviates from its assigned cluster, which enables us to distinguish significant outliers from temporary anomalies.
In the final stage, all algorithms are applied to two real-world trajectory datasets to analyze parameter sensitivity and assess the effect on the number of clusters. Three evaluation criteria are then employed to measure clustering performance before and after running the stable trajectory clustering algorithm: the Silhouette score evaluates the quality of the internal cluster, and Normalized Mutual Information (NMI), together with the Adjusted Rand Index (ARI), provides external benchmarks for comparison with the traditional and deep learning methods. The results show that the stable trajectory clustering algorithm efficiently manages insignificant anomalies to improve the quality of clusters and achieve better clustering accuracy compared to other approaches.
In summary, our contributions are illustrated as a graphical abstract in Fig. \ref{fig:100}. 
\par
The remainder of this paper is organized as follows: Section 2 covers the concepts and related algorithms for data point clustering, sub-trajectory clustering, whole-trajectory clustering, and stable data point clustering. Section 3 covers the foundational concepts and background information necessary for understanding our work, and then formally defines the problems to be solved. Section 4 addresses how whole-trajectories and sub-trajectories are clustered using the split and merge approaches. 
Section 5 introduces the principles and definitions of the stable trajectory clustering algorithm. Section 6 evaluates the algorithms and discusses their efficiency. Finally, Section 7 concludes the paper with remarks and directions for future work.
 \begin{figure}[htbp]
    \centering
    \includegraphics[width=.75\textwidth]{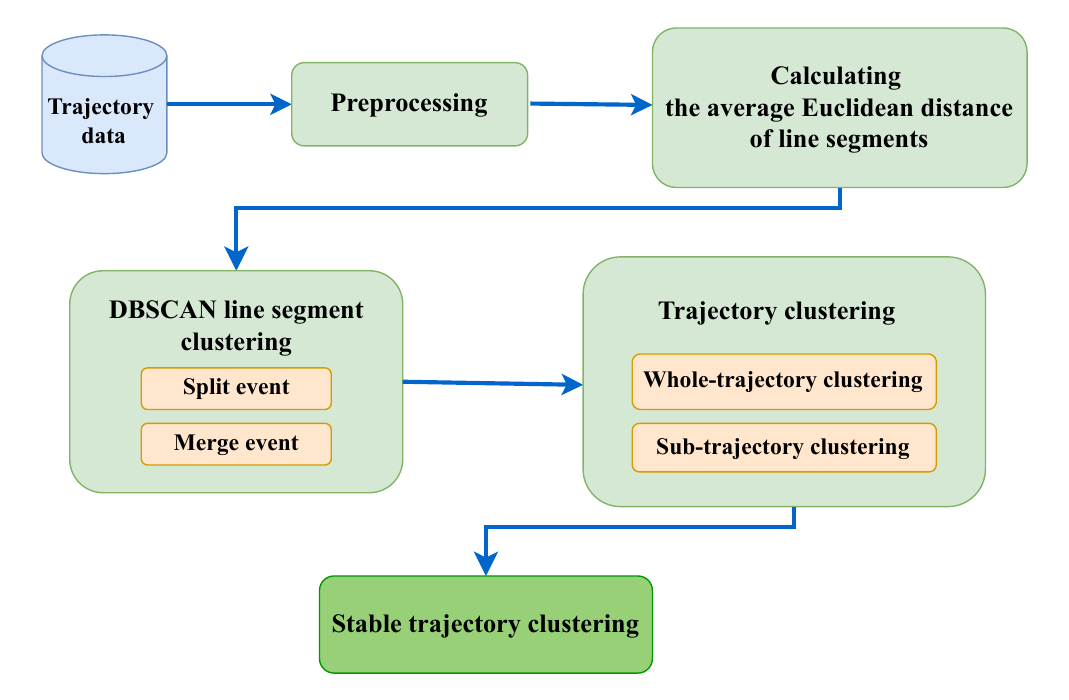}
    \caption{Graphical abstract of the proposed approaches.}
    \label{fig:100}
\end{figure}
\section{Related Work}
In the last two decades, many data point clustering algorithms have been discussed in various fields and applications such as wireless sensor networks~\cite{32}, flight anomaly detection~\cite{33}, and image processing~\cite{34}. These algorithms try to group similar data points based on their similarities or dissimilarities~\cite{23}. There are several types of methods to cluster data points, such as partitioning methods~\cite{25,24,26}, hierarchical methods~\cite{28,29}, grid-based methods~\cite{30,31}, and density-based methods~\cite{22,27}. Some of these methods are used for clustering data points as well as for trajectory clustering. For example, DBSCAN~\cite{22} is a density-based algorithm that is used for clustering both data points and trajectories. It can identify the shape and size of data, can handle noise and outliers efficiently, and does not require the number of clusters as a user-defined input. However, DBSCAN cannot identify clusters with different densities and is sensitive to the choice of user input parameters for clustering. Based on the algorithms, researchers then started developing new methods for clustering continuous movement objects, known as trajectory clustering.
\par
Trajectory data are recorded in different formats by various devices such as Global Positioning System (GPS) trackers, Radio Frequency Identification (RFID) systems, and smartphones. Numerous algorithms have been designed to cluster these different types of trajectory data. They also contribute to improving research, innovation, and theoretical understanding in various fields and applications. 
Trajectory clustering can be divided into point-based, sub-trajectory-based, and whole-trajectory-based methods. 
Point-based trajectory clustering is a method of clustering trajectories by comparing specific locations (points) along each trajectory. Instead of looking at the entire trajectory, it focuses on the individual points where the object was at different times and groups similar ones together.
Li et al.~\cite{12} offered a group discovery framework to facilitate the efficient online discovery of moving objects traveling collectively. Their framework employed a sampling-independent approach while also naturally supporting the utilization of approximate trajectories. Jensen et al.~\cite{11} investigated effectively managing a collection of constantly moving data points in a two-dimensional Euclidean space to maintain a cluster of dynamic data points. Their proposed approach can cluster moving objects incrementally. Tang et al.~\cite{10} introduced the smart-and-close discovery algorithm, which efficiently generates companions from trajectory data. In their approach, the objective of the proposed traveling buddy model was to improve the clustering and intersection processes involved in discovering companions.
Li et al.~\cite{9} recommended utilizing cluster analysis for mobile objects and introduced the notion of moving micro-clusters, which can effectively identify patterns in the mobility of objects even in massive datasets. Gao et al.~\cite{48} addressed the challenges of clustering real-time trajectory data generated by GPS-equipped devices. They proposed a novel framework that dynamically models trajectory streams using a self-adaptive k-segment method, considering both spatial and temporal dimensions. This framework supported distributed indexing and processing to handle large-scale data efficiently. It utilized incremental DBSCAN clustering for real-time processing and integrated a two-layered distributed index to enhance scalability. 
The study's key contributions included effective trajectory stream modeling, efficient and incremental clustering algorithms, and scalable distributed processing techniques. 

\par 
Sub-trajectory clustering involves breaking down a trajectory into smaller line segments and clustering them based on similarity in terms of spatiotemporal properties. It can be useful when it is not necessary to analyze the entire trajectory. In contrast, whole-trajectory clustering is more commonly used for identifying patterns over a longer period of time. Lee et al.~\cite{16} introduced a novel partition-and-group framework for clustering sub-trajectories, which involves dividing trajectories into a collection of line segments and subsequently grouping similar line segments into a cluster~\cite{13}. However, their framework does not consider temporal labels or the evolution of clusters across time, and therefore does not capture split and merge dynamics in trajectory clustering. Yu et al.~\cite{17} suggested a density-based clustering algorithm for trajectory data streams. They designed an algorithm to handle moving object trajectory data as an incremental line segment stream, using a TC-Tree structure to update clusters in real-time upon the arrival of each new segment. This is distinct from our interval-based approach, which analyzes the evolution of entire clusters between discrete time intervals. Da et al.~\cite{38} dealt with the problem of identifying sub-trajectory clusters at every time interval. To this aim, they defined a structure called a micro-group to take and maintain small groups of moving objects at each time window and presented an algorithm for maintaining each micro-group incrementally to discover patterns. Thus, they suggested a density-based sub-trajectory cluster identification approach by merging micro-groups. Buchin et al.~\cite{14} presented a linear-time algorithm for calculating the most similar sub-trajectories under equal starting times, while for unequal starting times, they offered $(1+\varepsilon)$-approximation algorithms that were not as efficient. This approach employed a distance measure defined as the average Euclidean distance at corresponding times. Mao et al.~\cite{36} used a sliding window model to propose an online algorithm for clustering streaming trajectories. This model contained a micro-clustering that represented clustering trajectory line segments in a current window and a macro-clustering based on micro-clustering over a specified time interval. They defined two data structures for tracking the latest alterations of a cluster trajectory data stream in real-time. Buchin et al.~\cite{15} also examined the problem of identifying commuting patterns in trajectories by searching for comparable sub-trajectories. They proposed multiple approximation algorithms and demonstrated that computing and approximating the longest sub-trajectory cluster was as hard as the Max-Clique problem. Liu et al.~\cite{35} defined a method for determining the distance between sub-trajectories considering time, space, and direction. They proposed a density trajectory cluster algorithm in which the cluster radius is calculated by the density of the data distribution. Schreck et al.~\cite{37} suggested a complete visual-interactive monitoring and control framework for trajectory data analysis based upon an enhanced version of the Self-Organizing Maps algorithm. Their framework allowed users to observe clustering progression visually and manage the algorithm at any desired level. Yang et al.~\cite{1} designed a trajectory clustering algorithm based on spatio-temporal density analysis of data and defined two new metrics-NMAST density function and NT factor- in this algorithm. Binh Han et al.~\cite{45} proposed a systematic method for whole-trajectory clustering of mobile objects in road networks through their framework TRACEMOB. This includes a unique distance measure that captures the complex characteristics of trajectories, an algorithm to transform these trajectories into multidimensional data points in Euclidean space, and a validation method for assessing clustering quality. The validation approach evaluates the quality of clustering in both the transformed metric space and the original road network space, ensuring comprehensive analysis. Wang et al.~\cite{21} designed a grid-based whole-trajectory clustering algorithm that converts the trajectories in the road network space into corresponding grid sequences within grid space. Their approach could determine a series of clusters as well as some abnormal trajectories and GPS points.
\par
In addition to the aforementioned methods, some studies have used deep learning for trajectory clustering. Yao et al. \cite{59} proposed a deep learning based approach to improve trajectory clustering. They used a sequence-to-sequence autoencoder to turn each trajectory into a short vector, which represents its overall movement pattern. Their method then clusters these simple vectors instead of the complex raw trajectories. An advantage of this method is that it allows the model to group trajectories with similar shapes, even if they aren't perfectly aligned in time or have different lengths.
Furthermore, Wang et al.\cite{60} introduced an end-to-end framework that learns how to represent trajectories and cluster them at the same time. Their method, called DTC, also relies on a sequence-to-sequence autoencoder that turns each trajectory into a vector. In other words, the model is trained to do two things at once: learn a good vector for each trajectory, and also encourage the vectors of similar trajectories to be closer together into clusters. The main new idea is that the model doesn’t work in two steps. Instead, it learns the representations and groups them into clusters at the same time, which helps improve the final result. Building on this end-to-end approach, Wang et al.\cite{61} proposed a more advanced framework for spatiotemporal trajectory clustering called DSTC. Their work argues that accurate clustering requires a more sophisticated representation of both space and time. To achieve this, for space, it replaces the rigid grid with a flexible, density-based approach that adapts to how data is distributed. For time, it introduces polar coordinates to correctly represent its cyclical nature. Second, it explicitly models the cyclical nature of time using polar coordinates. These improvements lead to state-of-the-art clustering accuracy. However, achieving this performance requires a much more complex model that is heavily reliant on very large datasets.
\par
Although clustering techniques have advanced, many methods still find it hard to deal with irregularities in trajectories.
Anomaly detection in trajectories involves identifying patterns and behaviors that deviate significantly from the norm~\cite{19,46,50,62}. Though the deviations are big enough to qualify as statistical outliers, they are ignorable with respect to the inference on the overall clustering pattern. 
This is what we address in our work via the stable trajectory clustering algorithm, where the trajectories remain in their respective clusters to the extent possible despite minor anomalies in a few time intervals. However, in previous studies, stable clustering algorithms refer to methods that create and maintain clusters of data points to minimize changes in cluster membership over time. These algorithms ensure efficient communication and resource utilization while considering factors such as distance and connectivity. Shahwani et al.~\cite{47} proposed a novel clustering algorithm for Vehicular Ad hoc Networks (VANETs) that enhances cluster stability by utilizing vehicle trajectories, specifically incorporating the traffic regularity of buses, into the clustering process. The algorithm used the Affinity Propagation technique, which traditionally relies on distance metrics, but modified it to also consider the shared trajectories of vehicles. By doing so, the algorithm selects Cluster Heads (CHs), which are central nodes in VANETs responsible for coordinating communication within a cluster of vehicles. These CHs not only have minimal distance from other vehicles but also share longer common trajectories with them, leading to more stable clusters. Simulation results demonstrated that this trajectory-based approach significantly improved cluster lifetime and reduced the frequency of cluster head changes compared to conventional methods, thereby enhancing the overall efficiency and reliability of VANET communications. Tseng et al.~\cite{54} suggested a stable clustering algorithm called the Clustering Algorithm using the Traffic Regularity of Buses for VANETs in urban scenarios. The algorithm leveraged the fixed routes and regular schedules of buses to improve cluster stability. It considered vehicle mobility characteristics such as velocity, position, and direction, along with the spatial and temporal dependencies of traffic patterns in urban areas. CATRB used the concept of an equilateral triangle's centroid, incenter, and circumcenter to select the most appropriate CH. This approach aimed to position the CH at the center of the cluster for efficient control and data transmission. These studies just focused on providing stable algorithms for data point clustering. While stable data point clustering is undoubtedly important, it does not fully address the complexities and dynamic nature of trajectory clustering. Our work aims to fill this gap by introducing and rigorously evaluating a stable algorithm designed for trajectory clustering.

\section{Preliminaries and Definitions}
In this section, we will start by defining key concepts such as trajectory, trajectory line segment, and sub-trajectory, which are essential for formally defining both the whole-trajectory clustering and sub-trajectory clustering problems. Then we will introduce the average Euclidean distance and explain the concepts underpinning the DBSCAN line segment clustering approach. 

\begin{definition}[Trajectory] Geometrically, a trajectory is a piecewise linear function of a moving object that is obtained by connecting geographical points that are recorded at different moments 
  $t_{1},t_{2},...,t_{T}$ by a tracking device. So, a trajectory for the $k$-th object is denoted as: 
$$\pi_{k}=(\overrightarrow{P_{k}^{1}P_{k}^{2}},\overrightarrow{P_{k}^{2}P_{k}^{3}},...,\overrightarrow{P_{k}^{T-1}P_{k}^{T}}),$$
such that each $P_{k}^{i}$ is the object's position at the $i$-th time-stamp (i.e., $P_{k}^{i}=\left(x_{k}\left(t_{i}\right),y_{k}\left(t_{i}\right)\right)$ where $x_{k}\left(t_{i}\right)$ and $y_{k}\left(t_{i}\right)$ are the longitude and latitude respectively).
\end{definition}
\begin{definition}[Trajectory line segment]
    $\overrightarrow{P_{k}^{i}P_{k}^{i+1}}$ is a trajectory line segment known as a line segment which is an approximate continuous object movement from position $P_{k}^{i}$ to $P{_k}^{i+1}$ with constant velocity and zero acceleration in a time interval $\left[t_{i},t_{i+1}\right]$.
\end{definition} 
\begin{definition}[Sub-trajectory] A sub-trajectory of trajectory 
 $\pi_{k}$ is characterized by the sub-sequence denoted as $\pi_{k}^{\left[i,j\right]}=(\overrightarrow{P_{k}^{i}P_{k}^{i+1}},\overrightarrow{P_{k}^{i+1}P_{k}^{i+2}},...,\overrightarrow{P_{k}^{j-1}P_{k}^{j}})$ where $1\leq i < j \leq T $. This time interval range is represented by $R_{p}$ with $p=\left[i,j\right]$. 

\end{definition}
\begin{problem}\label{4}
(Whole-trajectory clustering) Given $n$ trajectories as 
$\Pi=\left\lbrace \pi_{1},\pi_{2},...,\pi_{n}\right\rbrace$, 
the aim is to identify patterns and similarities among these trajectories and group them based on their spatiotemporal characteristics into clusters 
$C=\left\lbrace c_{1},c_{2},...,c_{\#cl} \right\rbrace$.
To this end, it is important to know the movement of the objects from time $t_1$ to $ t_T$. However, due to changes in the position of objects over a time interval, it is impossible to identify object movements solely based on their positions at specific moments $t_1,t_2,...,t_T$. Thus, for $n$ trajectories, there are $n$ line segments at any time interval, which should be clustered by an appropriate similarity criterion. The line segment trajectory clustering results at each time interval contribute to the whole-trajectory clustering of $\Pi$. 
\end{problem}
\begin{problem}
(Sub-trajectory clustering) \label{10}
    Let $\Pi=\left\lbrace \pi_{1},\pi_{2},...,\pi_{n}\right\rbrace$ be a set of $n$ trajectories. The purpose is to detect patterns and similarities among these trajectories in $T'(< T)$ time intervals such as $R_{1}, R_{2},..., R_{T'}$.  The resemblances can be obtained through $C_{p}^{R_{p}}= \left\lbrace c_{s_{1}},c_{s_{2}},...,c_{s_{num}}\right\rbrace$, in which $c_{s_{k}}$ is the clustering of sub-trajectories $\pi_{1}^{R_p},\pi_{2}^{R_p},...,\pi_{n}^{R_p}$ in time interval $R_{p}$. Thus, the set $C_{s}=\left\lbrace C_{1}^{R_{1}},C_{2}^{R_{2}},...,C_{T'}^{R_{T'}} \right\rbrace$ is obtained as a series of sets of clustered sub-trajectories.
\end{problem}

Although trajectory dissimilarity can be measured using various metrics (e.g., Fréchet distance, DTW \cite{58}), we adopt the average Euclidean distance because it is the most appropriate choice for piecewise-linear trajectories and fits our clustering framework, rather than being selected merely for convenience or because our algorithm only works with this metric. The distance is defined as follows:

\begin{definition}[Average Euclidean distance of line segments]\label{7} This measure should consider the speed and time of the moving object in a line segment. Hence, the dissimilarity between the two line segments \( l_{k} \) and \( l_{j} \) is defined as:
\begin{equation}\label{012}
 dist(l_k, l_j) = \frac{1}{\tau} \int_{\tau} \|P^t_k - P^t_j \| \, dt, 
\end{equation}
where \( P^t_k = (x_k(t), y_k(t)) \) and \( P^t_j = (x_j(t), y_j(t)) \) represent the positions of the objects at time \( t \) on the line segments \( l_k \) and \( l_j \), respectively. The notation \(\| \cdot \|\) denotes the Euclidean distance between the points at time \( t \), and \(\tau = t_{i+1} - t_i \) represents the time interval.
For any arbitrary \( t_{\sigma} \in [0,1] \), there is a geographical position \( P_{k}^{t_{\sigma}} \) on line segment \( l_{k} \), which is defined by the convex set as follows. This formula also holds for \( P_{j}^{t_{\sigma}} \).
\begin{equation}\label{1}
 \begin{cases}
x_{k}\left(t_{\sigma}\right)=x_{k}\left(t_{i+1}\right)\times t_{\sigma}+x_{k}\left(t_{i}\right)\times\left(1-t_{\sigma}\right),
\\
y_{k}\left(t_{\sigma}\right)=y_{k}\left(t_{i+1}\right)\times t_{\sigma}+y_{k}\left(t_{i}\right)\times \left(1-t_{\sigma}\right).
\end{cases}   
\end{equation}
\end{definition}
\begin{lemma}
The definition \eqref{012} satisfies all metric space axioms, including non-negativity, symmetry, and the triangle inequality.
\end{lemma}
\begin{proof}
To prove the non-negativity of the metric, we expand the definition as follows:

  The Euclidean distance between two moving objects $k$ and $j$ at any moment $t_{\sigma}$, is defined as:
\begin{equation}\label{2}
\begin{aligned}
    dist \left( P_{k}^{t_{\sigma}},P_{j}^{t_{\sigma}} \right) & =\sqrt{\left(x_{k}\left(t_{\sigma}\right)-x_{j}\left(t_{\sigma}\right)\right)^2 + \left(y_{k}\left(t_{\sigma}\right)-y_{j}\left(t_{\sigma}\right)\right)^2}.
    \end{aligned}
\end{equation}
Substituting the definitions of $x_{k}\left(t_{\sigma}\right)$ and $y_{k}\left(t_{\sigma}\right)$ from formula \eqref{1} into Eq. \ref{2}, we have:
 \begin{flalign*}
     dist \left( P_{k}^{t_{\sigma}},P_{j}^{t_{\sigma}} \right) & 
    =\sqrt{\left( At_{\sigma}+B\right)^2+ \left( Ct_{\sigma}+D\right)^2}\\
    & = \sqrt{A^2t_{\sigma}^2+2ABt_{\sigma}+B^2+C^2t_{\sigma}^2+2CDt_{\sigma}+D^2}\\
    &=\sqrt{\left(A^2+C^2\right)t_{\sigma}^2+\left( 2AB+2CD\right)t_{\sigma}+B^2+D^2},
\end{flalign*}
where $A=x_{k}\left(t_{i+1}\right)-x_{j}\left(t_{i+1}\right)-x_{k}\left(t_{i}\right)+x_{j}\left(t_{i+1}\right)$, $B =x_{k}\left( t_{\sigma}\right)- x_{j}\left( t_{\sigma}\right)$, $C=y_{k}\left(t_{i+1}\right)-y_{j}\left(t_{i+1}\right)-y_{k}\left(t_{i}\right)+y_{j}\left(t_{i+1}\right)$ and $D=y_{k}\left( t_{\sigma}\right)- y_{j}\left( t_{\sigma}\right)$.
Therefore, the (continuous) average Euclidean distance of two line segments, $l_{k}$ (i.e., $l_{k}=\overrightarrow{P_{k}^{i}P_{k}^{i+1}}$), and $l_{j}$ is defined as follows: 
\begin{equation} dist(l_{k},l_{j})=\frac{1}{\tau}\int_{t_{i}}^{t_{i+1}}\sqrt{\left(A^2+C^2\right)t_{\sigma}^2+\left( 2AB+2CD\right)t_{\sigma}+B^2+D^2}\,dt_{\sigma}.
\end{equation}
If we set $a=A^2+C^2,\quad b=2AB+2CD,\quad c=B^2+D^2$, then
\begin{align*}
  dist(l_{k},l_{j})&=\frac{1}{\tau}\int_{t_{i}}^{t_{i+1}}\sqrt{at_{\sigma}^2+bt_{\sigma}+c}\,dt_{\sigma}\\ 
  &=\frac{1}{\tau}\eval{\left(\frac{(2at_{\sigma}+b)\sqrt{a t_{\sigma}^2+bt_{\sigma}+c}}{4a}\right)}_{t_{i}}^{t_{i+1}}\\
  &-\eval{\frac{1}{\tau}\left(\frac{(b^2-4ac)\ln(2\sqrt{a} \sqrt{at_{\sigma}^2+bt_{\sigma}+c}+2at_{\sigma}+b)}{8a^{\frac{3}{2}}}\right)}_{t_{i}}^{t_{i+1}}.
\end{align*}

Since \( a > 0 \) and \( b^2 - 4ac < 0 \), the quadratic function \( at^2 + bt + c \) is positive definite, ensuring the metric is always non-negative. The identity of indiscernibles and symmetry are evident from the definition of the Euclidean norm in the integral. For the triangle inequality, given three line segments \( l_i \), \( l_j \), and \( l_k \), it can be shown straightforwardly as follows:
\[
d(l_{i}, l_{k}) \leq d(l_{i}, l_{j}) + d(l_{j}, l_{k}).
\] 
\end{proof}

Line segments can be clustered based on density in any time interval. In this paper, we use DBSCAN line segment clustering, which needs two user-predefined parameters, $\varepsilon$ and $MinLns$. The first one is the maximum distance between two line segments, which is used to locate the line segment in the neighborhood of any other line segment. The second one is the minimum number of line segments that need to be clustered together in a region to be considered dense. Assume that $L=\left\lbrace l_{1},l_{2},...,l_{n_{s}} 
\right\rbrace$ is the set of unclustered line segments. The principles and concepts of DBSCAN clustering for these line segments are defined as follows.
\begin{definition}[$\varepsilon$-neighborhood]
    The \textit{$\varepsilon$-neighborhood} of line segment $l_{i} \in L$ denoted by $N_{\varepsilon}(l_{i})$ and defined as $N_{\varepsilon}(l_{i})=\left\lbrace \forall l_{j} \: | \: dist(l_{i},l_{j})\leq  \varepsilon \right\rbrace, l_{j}\in L$.
\end{definition}
\begin{definition}[Core line segment]
    $l_{i}$ is a \textit{core line segment} if $ |N_{\varepsilon}(l_{i})| \geq MinLns $, where $|N_{\varepsilon}(l_{i})|$ is the size of set $N_{\varepsilon}(l_{i})$.
\end{definition}
\begin{definition}[Directly density-reachable]\label{14}
    $l_{i}$ is \textit{directly density-reachable} from a line segment $l_{j}$ if $l_{i}\: \in N_{\varepsilon}(l_{j})$ and $MinLns \leq |N_{\varepsilon}(l_{j})|$.
    \end{definition}
\begin{definition}[Density-reachable]\label{13}
        A line segment $l_{i}$ is \textit{density-reachable} from line segment $l_{j}$ if there is a chain of line segments $l_{i},l_{i+1},...,l_{j}$ such that line segment $l_{k}$ is directly density-reachable from line segment $l_{k+1}$ ($i \leq k \leq j-1$).
\end{definition}
\begin{definition}[Density-connected]\label{91}
    A line segment $l_{i}$ is \textit{density-connected} to line segment $l_{j}$ if there is a line segment $l_{k}$ such that both $l_{i}$ and $l_{j}$ are density-reachable from $l_{k}$.
\end{definition}
\begin{definition}[Borderline segment]\label{90}
    A line segment $l_{i}$ is a \textit{borderline segment} if it is not a core line segment but is directly density-reachable from a core line segment $l_{j}$.
\end{definition}
\begin{definition}[Outlier line segment]\label{12}
    A line segment $l_{i}$ is an \textit{outlier line segment} if it is neither a core line segment nor a borderline segment.
\end{definition}
\begin{definition}[Cluster of line segments]\label{CLS}
    Assume that $\Gamma$ is a non-empty subset of $L$. For $\Gamma$ to be a cluster, two conditions must be satisfied.
    \begin{enumerate}
        \item (\textit{Maximality}) For any pair of line segments $l_{i}$ and $l_{j}$, if $l_{i}       \in \Gamma$ and $l_{j}$ is density-reachable from               $l_{i}$, then $l_{j} \in \Gamma$.
        \item (\textit{Connectivity}) For any pair of line segments $l_{i}$ and $l_{j}$ which are in $\Gamma$, $l_{i}$ is density-connected to $l_{j}$. 
    \end{enumerate}
\end{definition}
The definitions above help clarify the key concepts and terminology relevant to our discussion of the stable trajectory clustering algorithm. The next section will explore how line segments can be clustered in each time interval, taking into account their history of clustering in the previous time interval. We then use the results of line segment clustering in each time interval to cluster whole-trajectories and sub-trajectories. 

\section{Clustering Scheme}
 As mentioned in previous clustering studies, the \textit{split} and \textit{merge} operations are two significant events for moving object clustering ~\cite{9,11}. In this paper, we also support this view of clustering line segments in each time interval.
 The split event for line segments involves identifying which line segments are separated from each other, while the merge event involves combining line segments that are close to each other and have similar properties. 
 
 Algorithm \ref{alg:1} outlines this trajectory clustering scheme. Segmentation is performed to simplify the analysis and processing of trajectories. The algorithm uses DBSCAN line segment clustering concepts to cluster these line segments within each time interval by the split and merge events of line segments. It checks each cluster from the previous time interval for splitting or merging. If a cluster is dense, it may be split; if clusters are mergeable, they may be combined. Finally, it utilizes the line segment clustering results for both whole-trajectory and sub-trajectory clustering. 
 
\begin{algorithm}[ht]
\newcommand\mycommfont[1]{\small\ttfamily\textcolor{black}{#1}}
\SetCommentSty{mycommfont}
\SetKwData{Left}{left}
\SetKwData{This}{this}
\SetKwData{Up}{up}
\SetKwFunction{Union}{Union}
\SetKwFunction{FindCompress}{FindCompress}
\SetKwInOut{Input}{input}
\SetKwInOut{Output}{output}
\caption{Trajectory clustering}
\label{alg:1}
\Input{$\Pi$, $ \varepsilon, MinLns$, $S$ and $W$.}
\Output{Whole-trajectory cluster set $C$ and sub-trajectory cluster set $C_{s}$. }
$m  \leftarrow T-1;$\\
divide each $\pi_{i} \in \Pi$ into $m$ parts;\\
$C_{curr} \leftarrow DBSCAN(L,\varepsilon,MinLns);$\\
$C_{hist}\left\lbrace1\right\rbrace \leftarrow \left\lbrace C_{curr}\right\rbrace;$\\
\For{$t$ from 2 \KwTo $m$}
{
$C_{new} \leftarrow \emptyset;$\\

\ForEach{$c_i \in C_{curr}$}
{
$C_{new}$.union$\big( |c_i| > 1 \,?\, \text{Split}(c_i, \varepsilon, MinLns):c_i \big);$
}
\ForEach{pairs of $c_{i}$ and $c_{j}$ in $C_{new}$}
{
\If{$isMergeable(c_{i}, c_{j})$}
{$C_{mrg} \leftarrow $Merge$(c_{i}, c_{j},\varepsilon,MinLns));$ \\
\If{$C_{mrg}\neq \emptyset$}
{remove $c_{i}$ and $c_{j}$;\\ $C_{new}$\{end+1\}$ \leftarrow C_{mrg}$;
}
}
}
$C_{curr} \leftarrow C_{new}; $\\
$ C_{hist}\left\lbrace t\right\rbrace \leftarrow \left\lbrace C_{curr}\right\rbrace;$
}
$C \leftarrow $ wholeTraj($C_{hist},m,\Pi$);\\
$C_{s} \leftarrow $ subTraj($C_{hist},m,\Pi,S,W$);\\
\Return $C,\: C_{s};$
\SetAlgoLined
\end{algorithm}
The details of this algorithm are as follows: Firstly, the preprocessing should be run on data trajectories to divide each trajectory into $m$ parts (lines 1 and 2). Then, in the first time interval, for obtaining the first clusters, DBSCAN line segment clustering should be executed (line 3). From the second time interval to the last time interval (lines 5-16), if any cluster of line segments exists, the Split function (Algorithm \ref{alg:2}) is called (line 8). Two outputs are obtained: either the segments remain in their cluster, or they are split from each other. The Merge function (Algorithm \ref{alg:3}) should then be executed based on these results (lines 9-15). Eventually, the results of line segment clustering are utilized as inputs for the whole-trajectory clustering (Algorithm \ref{alg:4}) and the sub-trajectory clustering (Algorithm \ref{alg:5}) algorithms (lines 17 and 18). 
Here, in the first time interval, the initialization cost is $\mathcal{O}(n^2)$ (or $\mathcal{O}(n\log n)$ with spatial indexing). For the remaining $m$ time intervals, let $\mathcal{C}_{split}$ and $\mathcal{C}_{merge}$ denote the computational costs of the split and merge operations, respectively, and let $k$ be the number of clusters in each interval. Thus, the total worst-case complexity for Split-Merge part (lines 1-16) is $\mathcal{O}(n^2 + m(k \cdot \mathcal{C}_{split} + k^2 \cdot \mathcal{C}_{merge}))$.
\subsection{Split of Line Segments}
According to definition \ref{CLS}, let $C_{s}$ be a cluster containing $ n_{s}$ line segments in the time interval $\left[t_{i-1},t_{i}\right]$, where $n_{s}\leq n$. In the following, we will provide a definition and explanations regarding the split of the cluster.
\begin{definition}[Split event]
    Cluster $C_{s}$ will be split into at least two non-empty sub-clusters $c_{s_{1}}$ and $c_{s_{2}}$ in the next time interval $\left[t_{i},t_{i+1}\right]$ if no line segment in $c_{s_{1}}$ is density-reachable from any line segment in  $c_{s_{2}}$. In other words, if we have two sub-clusters $c_{s_{1}}$ and $c_{s_{2}}$, and for every line segment $l_{i}$ in $c_{s_{1}}$ and every line segment $l_{j}$ in $c_{s_{2}}$, there is no chain of density-reachable line segments connecting $l_{i}$ to $l_{j}$, then $l_{i}$ and $l_{j}$ are not in the same cluster. Therefore, $c_{s_{1}}$ and $c_{s_{2}}$ are considered split sub-clusters.
\end{definition}
 Suppose that $UC_{s}$ is a set of line segments in the time interval $\left[t_{i},t_{i+1}\right]$ that have previously been in cluster $C_{s}$. To detect whether the cluster $C_{s}$ will be split in the next time interval or not, a random line segment in $UC_{s}$ should be selected initially, and then its neighboring line segments are determined based on their average Euclidean distance as sub-set $c_{s_{1}}$. The same process is repeated for the remaining line segments in $UC_{s}$ until all line segments are examined, and no line segments are remaining. So, there is a set of sub-sets $NC_{s}=\left\lbrace c_{s_{1}},c_{s_{2}},...,c_{s_{num}}\right\rbrace $ and we need to determine whether the members of each sub-set can be considered neighbors of each other or not. If there is no chain of line segments that connects two sub-sets, they are considered split from each other; otherwise, they are connected. So, the $NC_{s}$ should be updated such that its members are the new sub-sets where, according to definition \ref{CLS}, each sub-set satisfies the Maximality criterion. However, for each subset to be a cluster, we also need to check for Connectivity. Consequently, we have sub-sets that are either considered as a sufficiently dense cluster or are recognized only as a low-density cluster of $NC_{s}$ due to violating the Connectivity criterion. It is noteworthy that it is possible to preprocess the line segments before using formula \eqref{7}, meaning we can measure only the Euclidean distance between the start and end points of the line segments without using the average Euclidean distance of line segments. Therefore, the following cases occur between the start and end points of each pair of line segments in $UC_{s}$ at the time interval $\left[t_{i},t_{i+1}\right]$:
\begin{itemize}
    \item If the Euclidean distance of two start points and two endpoints is less than or equal to $\varepsilon$, both line segments are in the same sub-set and it is not necessary to calculate the average Euclidean distance of line segments.
    \item If the Euclidean distance of start points, endpoints, or both of them is more than  $\varepsilon$, calculating the average Euclidean distance of line segments is necessary.
\end{itemize}

Algorithm \ref{alg:2} provides the pseudocode for the split event used in trajectory clustering. The algorithm takes a set of densely clustered line segments in a time interval and determines whether these line segments remain connected or are split from each other in the next time interval. If they are split, the resulting sub-clusters are then evaluated to check if they are dense enough. This involves checking the distances between line segments to identify neighbors, ensuring maximal connectivity between clusters, and assessing the density of the clusters based on the presence of core line segments.
\begin{algorithm}[ht]
\newcommand\mycommfont[1]{\small\ttfamily\textcolor{black}{#1}}
\SetCommentSty{mycommfont}
\SetKwData{Left}{left}
\SetKwData{This}{this}
\SetKwData{Up}{up}
\SetKwFunction{Union}{Union}
\SetKwFunction{FindCompress}{FindCompress}
\SetKwInOut{Input}{input}
\SetKwInOut{Output}{output}
\caption{Split}
\label{alg:2}
\Input{Line segments set $UC_{s}=\left\lbrace l_{1},l_{2},...,l_{n_{s}}\right\rbrace$ in time interval $\left[t_{i-1},t_{i}\right]$, $\varepsilon$, $MinLns$.}
\Output{New non-empty sub-clusters $NC_{s}=\left\lbrace c_{s_{1}},c_{s_{2}},...,c_{s_{ncl}}\right\rbrace$ in time interval $\left[ t_{i},t_{i+1}\right]$. }
\tcp{Neighborhood analysis}
\ForEach{line segment $l_{i} \in UC_{s}$}
{$c_{s_{i}} \leftarrow \left \lbrace l_{i}\right\rbrace $;\\
\ForEach{line segment $l_{j} \in UC_{s}$}
{
\If{$dist(l_{i},l_{j}) \leq \varepsilon$}
{$c_{s_{i}}$.add$(l_{j})$;\\ }
$NC_{s} $.add($c_{s_{i}}) $\;
}
}
\tcp{\textbf{Maximality}} 
\ForEach{pair of $c_{s_{i}}$ and $c_{s_{j}}$ in $NC_{s}$ }{
  \text{separated-clusters} $\leftarrow$ \text{true};\\
  \ForEach{$l_{s_{i}} \in c_{s_{i}} $ }{
    \ForEach{$l_{s_{j}} \in c_{s_{j}} $ }{
      \If{$dist(l_{s_{i}},l_{s_{j}})\leq \varepsilon$}{
        \text{separated-clusters} $\leftarrow$ \text{false};\\
        \text{$c_{s_{i}}$ and $c_{s_{j}}$ are connected};\\
        goto end-loop;\\
      }
    }
  }
 {end-loop:}\\
  \If{\text{separated-clusters} is true}{\text{$c_{s_{i}}$ and $c_{s_{j}}$ are split};}
} 
\tcp{\textbf{Connectivity}}
\ForEach{$c_{s_{i}} \in NC_{s}$}
{\eIf{‌$c_{s_{i}}$ contains at least one core line segment}
{$c_{s_{i}}$ is a dense enough cluster;\\}{$c_{s_{i}}$ is an outlier or low-density cluster;\\}}
\Return $NC_{s}$\;
\SetAlgoLined
\end{algorithm}

The detailed steps of this algorithm are as follows: the algorithm finds the neighboring line segment for each line segment in $UC_{s}$ and 
creates clusters without determining whether they are sufficiently dense or not (lines 2-7). The algorithm then checks the Maximality (lines 9-19)
and Connectivity (lines 21-25) criteria, as mentioned in definition \ref{CLS}. The time complexity for checking the distances between all line segment pairs in $UC_s$ is
 $\mathcal{O}(|UC_s|^2)$.
\subsection{Merge of Line Segments}
Let $c_{s_{1}}$ and $c_{s_{2}}$ be two independent low-density clusters or sufficiently dense clusters of line segments that comprise at most $n_{s_{1}}$ and $n_{s_{2}}$ line segments, respectively, where $n_{s_{1}}+n_{s_{2}}\leq n$. The following definition explains how dense clusters, low-density clusters, or a combination of both can be merged. 

\begin{definition}[Merge event]
Two clusters $c_{s_{1}}$ and $c_{s_{2}}$ are merged
if there are at least two line segments $l_{i} \in c_{s_{1}}$ and $l_{j} \in c_{s_{2}}$ such that the average Euclidean distance of $l_{i}$ and $l_{j}$ is less than or equal to $\varepsilon$.
\end{definition}
The following scenarios may occur during the merge event.
\begin{itemize}
    \item Merging outlier line segments or low-density clusters.
    \item Merging outlier line segments or low-density clusters with dense clusters.
    \item Merging dense clusters with each other.
\end{itemize}
To identify the two line segments that lead to the merge of two clusters, a simple search is conducted between the elements of the two clusters. The search stops as soon as two line segments are found that have an average Euclidean distance of less than or equal to $\varepsilon$. 
The merge event may decrease the number of clusters, however, and determining whether new clusters have the required density depends on the user input parameters set for the DBSCAN clustering algorithm. 
 
Algorithm \ref{alg:3} merges two clusters of line segments based on a proximity threshold and evaluates whether the resulting cluster is dense. The density is determined by the presence of core line segments and the total number of line segments in the merged cluster.


\begin{algorithm}[ht]
\newcommand\mycommfont[1]{\small\ttfamily\textcolor{black}{#1}}
\SetCommentSty{mycommfont}
\SetKwData{Left}{left}
\SetKwData{This}{this}
\SetKwData{Up}{up}
\SetKwFunction{Union}{Union}
\SetKwFunction{FindCompress}{FindCompress}
\SetKwInOut{Input}{input}
\SetKwInOut{Output}{output}
\caption{Merge}
\label{alg:3}
\Input{Two non-empty low-density or dense enough clusters $c_{m_{1}}$ and $c_{m_{2}}$, 
$\varepsilon$, $MinLns$.}
\Output{Cluster $c_{m_{1}m_{2}}$. 
}
\ForEach{line segment $l_{i}$ in $c_{m_{1}}$}
{
\ForEach{line segment $l_{j}$ in $c_{m_{2}}$}
{
\If{$dist(l_{i},l_{j})\leq \varepsilon$}
{$c_{m_{1}m_{2}} \leftarrow c_{m_{1}} \cup c_{m_{2}} $\; goto Connectivity; \\}
}
}
\tcp{\textbf{Connectivity}}
\eIf{$|c_{m_{1}m_{2}}|=|c_{m_{1}}|+|c_{m_{2}}|$}
{
\eIf{$c_{m_{1}m_{2}}$ contains at least one core line segment }
{$c_{m_{1}m_{2}}$ is a dense enough cluster;}
{$c_{m_{1}m_{2}}$ is a low-density cluster;}
\Return $c_{m_{1}m_{2}};$}
{\Return $\emptyset$\;}
\SetAlgoLined
\end{algorithm}

In this algorithm, two sufficiently dense or two non-empty low-density clusters are given as inputs. Then it is checked which two line segments can connect two clusters (lines 1-6). If two clusters are connected, then the Connectivity should be checked (lines 7-14).
It is noteworthy that in some time intervals, both split and merge events occur.
The Split event never executes on the outlier line segments. Instead, only the Merge event is applied to them, which alters the role of the line segments in the clusters they join. In the worst case, the algorithm performs $\mathcal{O}(|cm_1| \cdot |cm_2|)$ comparisons.
 
\subsection{Whole-trajectory Clustering and Sub-trajectory Clustering}

In whole-trajectory clustering, the algorithm considers the entire trajectory of an entity. In other words, features like average Euclidean distance and constant velocity of moving objects in each time interval hold significance for clustering. With this in mind, we can use the clustered line segment information obtained by the split and merge events which we discussed in previous sections. Thus, based on definition \ref{4}, two whole trajectories $\pi_{i}$ and $\pi_{j}$ can be assigned to the same cluster if their line segments consistently belong to the same cluster; either they are neighbors of each other directly or indirectly across all time intervals (definitions \ref{13} and \ref{14}). 

Algorithm \ref{alg:4} provides whole-trajectory clustering by comparing trajectories at multiple time intervals to identify similarities. It counts the number of time intervals during which the line segments of each pair of trajectories are in the same cluster. If this count matches the total number of time intervals, the trajectories in the pair are considered similar because they consistently follow similar paths over time. The algorithm constructs a joint or disjoint graph that represents these similar relationships between trajectories and returns each connected component in this graph as a cluster of similar trajectories.

\begin{algorithm}[ht]
\newcommand\mycommfont[1]{\small\ttfamily\textcolor{black}{#1}}
\SetCommentSty{mycommfont}
\SetKwData{Left}{left}
\SetKwData{This}{this}
\SetKwData{Up}{up}
\SetKwFunction{Union}{Union}
\SetKwFunction{FindCompress}{FindCompress}
\SetKwInOut{Input}{input}
\SetKwInOut{Output}{output}
\caption{Whole-trajectory clustering}
\label{alg:4}
\Input{$C_{hist}$, $m$, $\Pi$.}
\Output{$C=\left\lbrace
c_{1},c_{2},...,c_{\#cl}\right\rbrace$.}
$\mathcal{A}_{\mathtt{traj}} \leftarrow \mathbb{O}_{m \times m};$\\
\ForEach{trajectory $\pi_{i} \in \Pi$ }
{
$counter \leftarrow 0 ;$\\
\ForEach{trajectory $\pi_{j} \in \Pi$ where $j>i$}
{\For{t from 1 \KwTo $m$}
{
\For{k from 1 \KwTo $|C_{hist}\lbrace t \rbrace|$}
{\If{$l_{\pi_{i}}$ and $l_{\pi_{j}} \in C_{hist}\lbrace t \rbrace \lbrace k \rbrace$ }
{$counter \leftarrow counter+1;$}
}
}
$\mathcal{A}_{\mathtt{traj}} (i,j) \leftarrow (counter=m)\; ? \; 1 : 0;$
}
}
$C\leftarrow$ connComp(makeGraph($\mathcal{A}_{\mathtt{traj}} $));\\
$\Return \;C;$
\SetAlgoLined
\end{algorithm}

In this algorithm, a matrix is used to determine which whole trajectories are in the same cluster in all time intervals (lines 4-9). Subsequently, the graph structure is employed to determine the whole-trajectory clusters. Each graph represents a cluster of whole trajectories. Besides that, a node corresponds to a trajectory, and the edge between two nodes signifies their neighborhood (line 10). The time complexity of the algorithm is $\mathcal{O}(n^2 \cdot m)$, because it compares the histories of line segment clusters for every trajectory pair over $m$ time intervals. Additionally, standard DFS or BFS computes the connected components of the resulting graph in $\mathcal{O}(n + n^2)$ time. 
\par
For sub-trajectory clustering, we use the sliding window model. This model is appropriate when we aim to detect specific ranges within a given dataset and assess the patterns, trends, or deviations within those ranges ~\cite{41,42,43}. As mentioned in the definition \ref{10}, the goal of sub-trajectory clustering is to recognize the time interval ranges like $R_{p}$ and cluster the sub-trajectories in this range as $C_{p}^{R_{p}}$. Utilizing the sliding window model as a tool can facilitate the process of finding these time interval ranges more easily. Therefore, we should first set the size of the sliding window and identify how many steps we want to move it. We consider two parameters $W$ and $S$ as the sliding window size and the movement step forward, respectively, which depend on the time interval range. To give an illustration, if each time interval range is considered 5 seconds, for $W=3$ and $S=1$, the sliding window encompasses the data of 3 consecutive time intervals, which are 15 seconds, and moves forward 5 seconds for each step (Fig. \ref{fig:5}). The window moves across all time intervals and according to the window size, contains the line segment clusters obtained in the previous section by the split and the merge events. In the current window, the consecutive time intervals that incorporate similar line segment clusters are determined as sub-trajectory clusters. Based on the parameter $S$, moving the sliding window either combines similar sub-trajectory clusters with the previous ones or creates new sub-trajectory clusters. If $S<W$, then the consecutive windows contain common time intervals and line segment clusters.

Algorithm \ref{alg:5} describes sub-trajectory clustering. 
It processes pre-clustered line segments to identify and cluster similar sub-trajectories. It employs a sliding window approach, iterating over the line segments according to parameters $S$ and $W$. For each window, sub-trajectories are extracted and clustered, with the results saved incrementally. After processing all windows, consecutive clustering results are compared to identify similar and different ranges. Continuous ranges of similar results are merged, while different ranges are noted separately. The algorithm then compiles these identified ranges into a final list of clustered sub-trajectories, resulting in a set of clusters that represent similar segments within the overall trajectory data. 

\begin{algorithm}[ht]
\newcommand\mycommfont[1]{\small\ttfamily\textcolor{black}{#1}}
\SetCommentSty{mycommfont}
\SetKwData{Left}{left}
\SetKwData{This}{this}
\SetKwData{Up}{up}
\SetKwFunction{Union}{Union}
\SetKwFunction{FindCompress}{FindCompress}
\SetKwInOut{Input}{input}
\SetKwInOut{Output}{output}
\caption{Sub-trajectory clustering}
\label{alg:5}
\Input{$C_{hist}$, $m$, $S$ and $W$.}
\Output{$C_{s}=\left\lbrace C_{1}^{R_{1}},C_{2}^{R_{2}},...,C_{m}^{R_{m}} \right\rbrace$.}
\For{i from 1 to $|C_{hist}|$+W-1 by step S}
{
$W_{s} \leftarrow i;$
$W_{e} \leftarrow i+W-1;$\\
\For{j from 1 to $|\Pi|$}
{$\Pi_{new}$.add$(\pi_{j}(W_{s}:W_{e}));$}
$C_{new} \leftarrow C_{hist} \lbrace W_{s}:W_{end}-1\rbrace;$\\
$p\leftarrow[W_{s},W_{e}];$\\
$C_{i}^{R_{p}}\leftarrow$ wholeTraj$(C_{new},R_{p},\Pi_{new})$;\\
$W_{temp}\lbrace i\rbrace \leftarrow C_{i}^{R_{i}}; $\\
}
$W_{move}\leftarrow |W_{temp}|;$\\
\For{j from 1 \KwTo $W_{\text{move}}-1$}{
  $\mathcal{B}_{\text{sim}} \gets (W_{\text{temp}}\{j\} = W_{\text{temp}}\{j+1\}) \; ? \; 1 : 0$\;
}
$C_{s} \leftarrow \emptyset;$\\   
\ForEach{$i \in \{i' \mid \mathcal{B}_{\text{sim}}[i'] = 0\}$} 
{
    $W_{s} \leftarrow (i-1) \times S + 1;$
    $W_{e} \leftarrow W_{s} + W;$ \\
    $p \leftarrow [W_{s}, W_{e}];$ \\
    $C_{s}.\text{add}(C_{p}^{R_{p}});$ 
}
$\mathcal{I}_{\text{sim}}\leftarrow$findContinuousOnes$(\mathcal{B}_{\text{sim}});$\\
\For{i from 1 to $|\mathcal{I}_{\text{sim}}|$}
{
$e\leftarrow $lastElement$(\mathcal{I}_{\text{sim}}\{i\})$;\\
$W_{s}\leftarrow($firstElement$(\mathcal{I}_{\text{sim}}\lbrace i\rbrace)-1)\times S+1;$\\
$W_{e} \leftarrow   ((e-1)\times S+1)+W;$\\
$p \leftarrow [W_{s},W_{e}];$\\
$C_{s}$.add$(C_{p}^{R_{p}});$
} 

\Return $C_{s};$
\SetAlgoLined
\end{algorithm}
The detailed steps of this algorithm are as follows: Firstly, based on the parameters $S$ and $W$, the trajectories within each window are clustered (lines 1-8). Then, each pair of continuous windows should be compared (lines 10 and 11).  
Continuous ranges with identical results are therefore merged and identified as a single range in which the sub-trajectories all have the same patterns (lines 13-23). The time complexity is $ \mathcal{O}\left( \frac{m}{S} \cdot (n^2 \cdot W) \right) $, because the algorithm repeats the whole-trajectory clustering process on a window of size $W$ for approximately $m/S$ steps.

Both whole-trajectory and sub-trajectory clustering follow the same procedures, utilizing the results of the split and merge events of line segment clustering in each time interval.
Nevertheless, while the whole-trajectory clustering algorithm focuses on patterns and similarities across entire trajectories, sub-trajectory clustering examines specific parts of trajectories via a sliding window.

\section{Stable Trajectory Clustering Algorithm}\label{sec:stb}
In this section, we address the concept of stability and how it can be applied in the proposed trajectory clustering algorithms.

\begin{definition}[Stable trajectory clustering]\label{9}
To further enhance the results of trajectory clustering, we can evaluate the outlier trajectories, which are the trajectories of 
 moving objects that deviate from their respective clusters for a few intervals and create anomalies. Such anomalies may be caused by noise in GPS trackers~\cite{20} or conditions that are dependent on the nature of the object's movement. Sometimes minor deviations are often inconsequential and can be disregarded. 
 Whether or not deviations should be ignored depends on the number of time intervals over which the objects’ trajectories remain split from their clusters. Additionally, the magnitude of their deviation from the clusters within these time intervals is important. Indeed, we aim to determine whether, despite the anomaly causing a trajectory to split from its cluster temporarily, it still belongs to the same cluster, or if this anomaly represents a genuine split.

As an example, the trajectories of six moving objects from top to bottom in different colors are shown in Fig. \ref{fig:1}. There are two vertical and horizontal crossed corridors in this example. Due to the mid-trajectory deviations where objects briefly enter the horizontal corridor, the sub-trajectory clustering and whole-trajectory clustering produce fragmented results, and split the trajectories into five and three clusters, respectively (Fig. \ref{fig:2} and Fig. \ref{fig:3}). 
In contrast, the stability of the trajectory can mitigate the impact of movements in the horizontal corridor, and different results may be achieved, as shown in Figs. \ref{fig:4} and \ref{fig:5}, respectively.
 \end{definition}

\begin{figure*}[htbp]
    \centering
    \begin{subfigure}{0.19\textwidth}
        \centering
        \includegraphics[width=\linewidth]{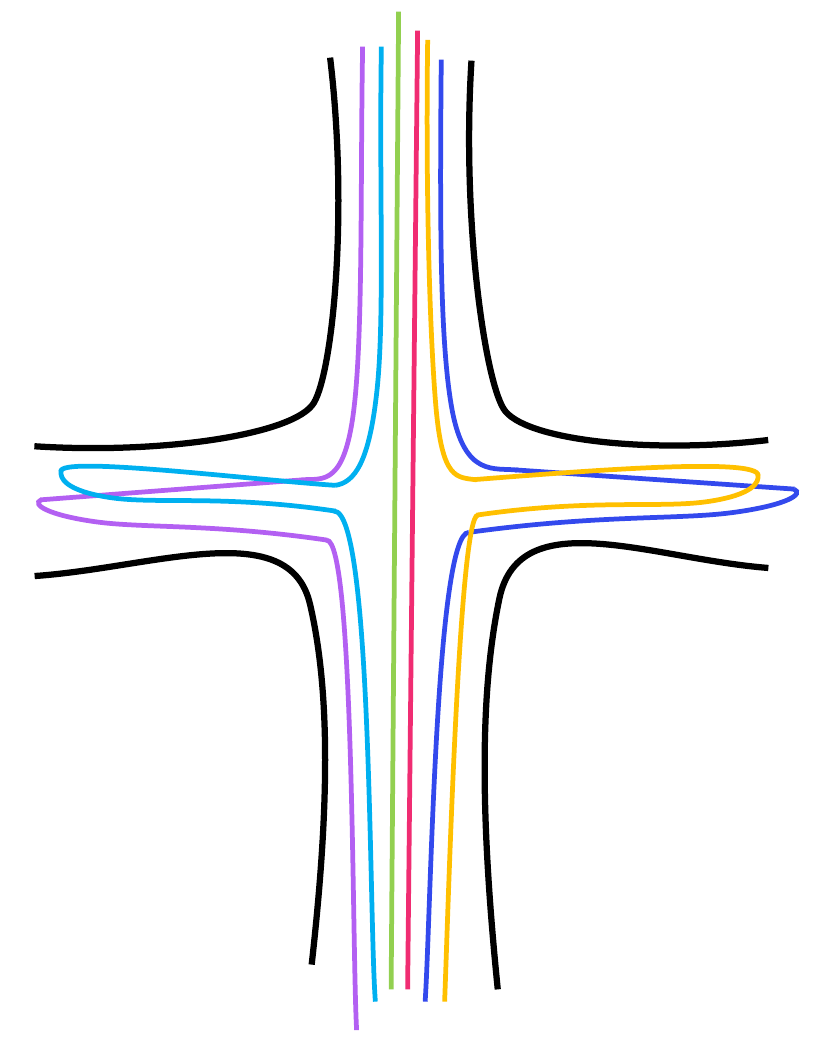}
        \caption{}
        \label{fig:1}
    \end{subfigure}
    \hfill
    \begin{subfigure}{0.19\textwidth}
        \centering
        \includegraphics[width=\linewidth]{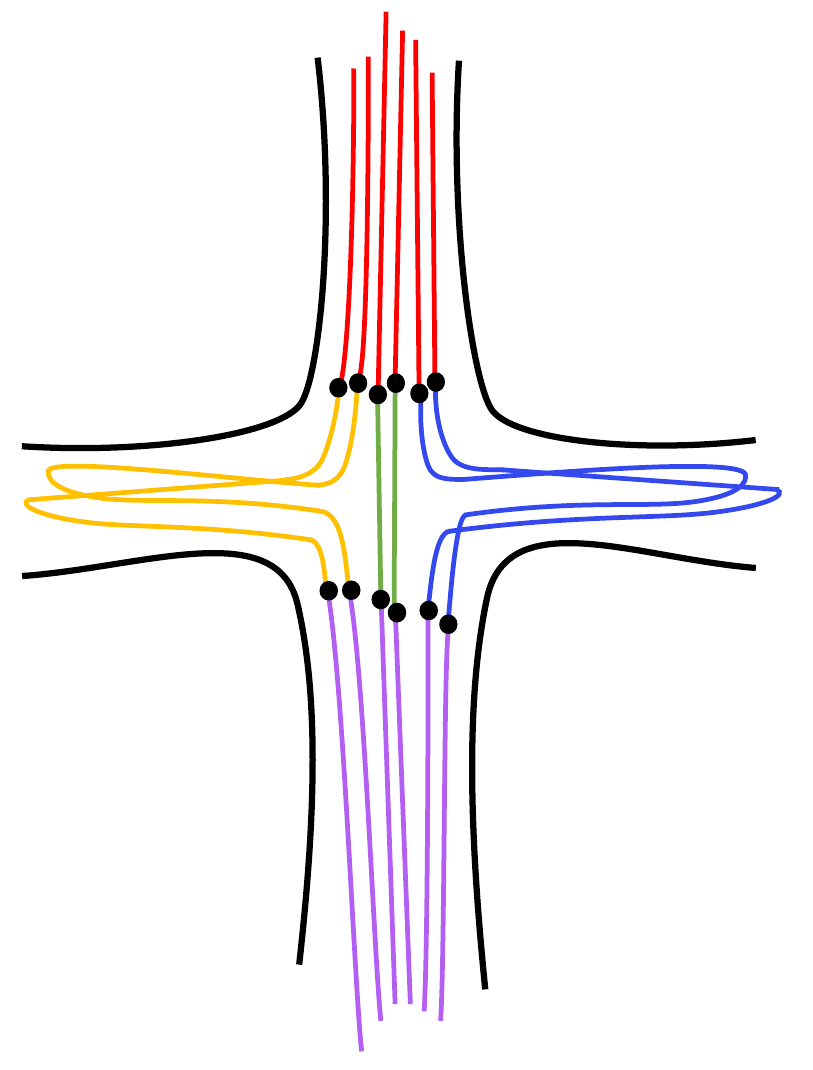}
        \caption{}
        \label{fig:2}
    \end{subfigure}
    \hfill
    \begin{subfigure}{0.19\textwidth}
        \centering
        \includegraphics[width=\linewidth]{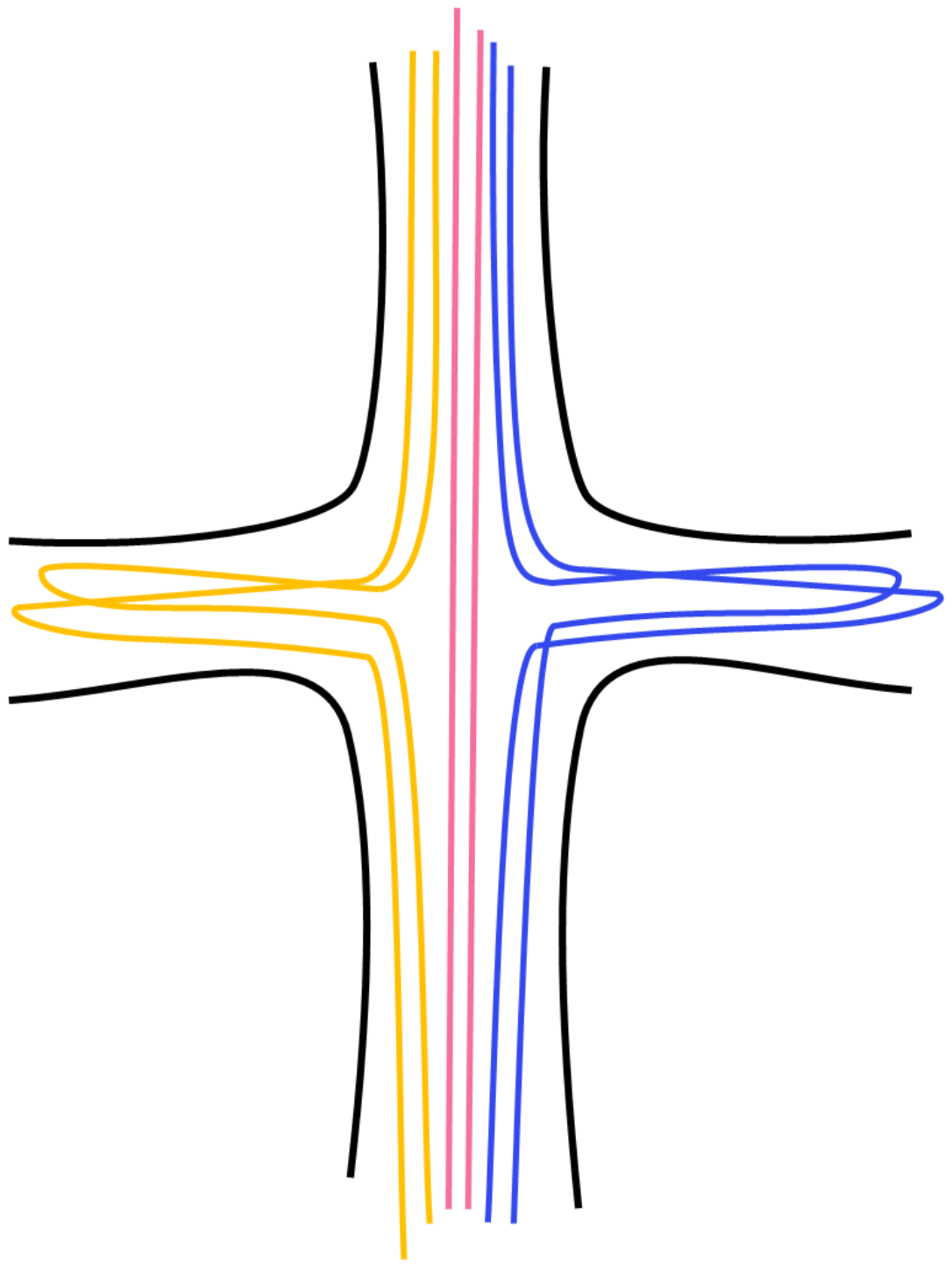}
        \caption{}
        \label{fig:3}
    \end{subfigure}
    \hfill
    \begin{subfigure}{0.19\textwidth}
        \centering
        \includegraphics[width=\linewidth]{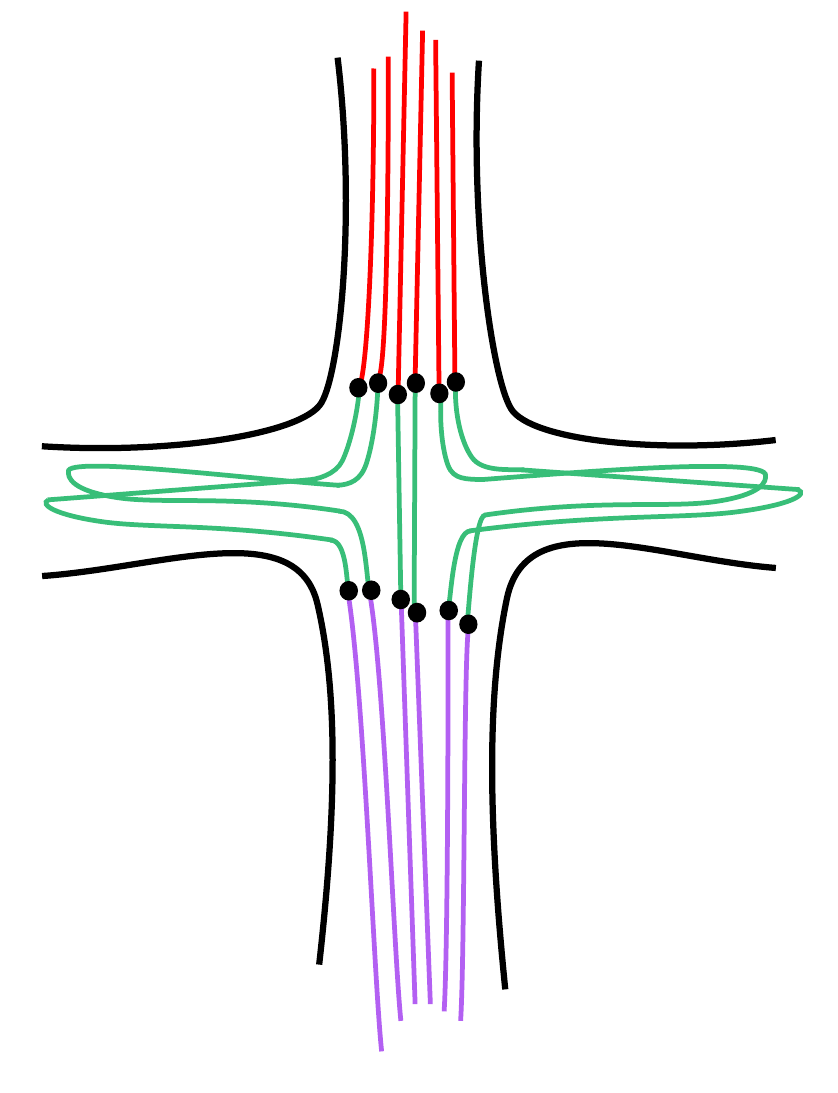}
        \caption{}
        \label{fig:4}
    \end{subfigure}
    \hfill
    \begin{subfigure}{0.19\textwidth}
        \centering
        \includegraphics[width=\linewidth]{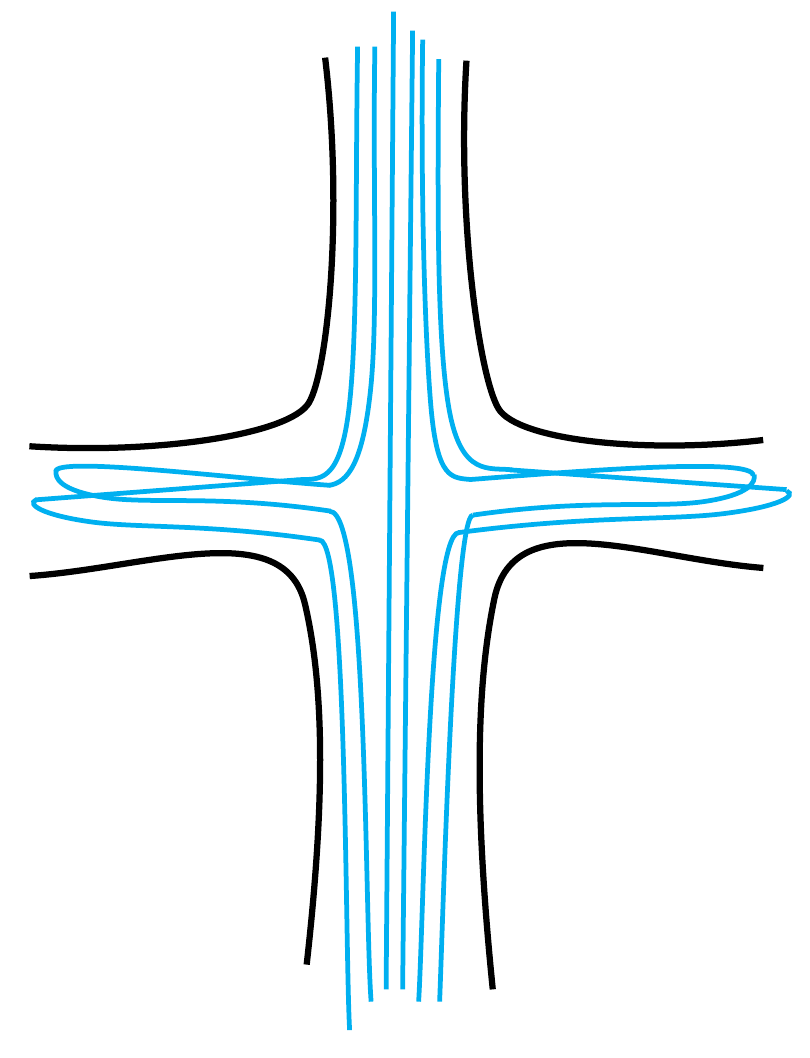}
        \caption{}
        \label{fig:5}
    \end{subfigure}
    
    \caption{Trajectory clustering based on the concept of stability. (a) The trajectories of six moving objects from top to bottom. (b) Detection of five clusters by sub-trajectory clustering. (c) Identification of three clusters by whole-trajectory clustering. (d) Disregarding the mid-trajectory anomalies and considering three sub-trajectories within a single cluster. (e) Ignoring the mid-trajectory anomalies completely and considering all trajectories as belonging to one cluster.}
    \label{fig:example}
\end{figure*}
\par
Building upon the discussions of whole-trajectory clustering and sub-trajectory clustering presented earlier, we can now obtain trajectory clusters and outliers. Then, the outlier or split trajectories can be assessed to determine whether their split should be considered a real split or not. 
Suppose $C_{num}$ represents a cluster of trajectories and $\pi_{a}$, $\pi_{b}$ are two outlier or split trajectories which, due to the anomalies in some time intervals, do not belong to the cluster $C_{num}$. Thus, they should be evaluated by the stable trajectory clustering algorithm. In the following, we will begin this evaluation with $\pi_{a}$. 
 We intend to determine the distance between the modified outlier trajectory of $\pi_{a}$ that is denoted as $\pi'_{a}$ and the members of cluster $C_{num}$ such as $\pi_{c}$ using definition \ref{7}. To this aim, suppose $l_{\pi_{a}}$ represents a line segment of $\pi_{a}$ and $l_{\pi_{c}}$ is a line segment of $\pi_{c}$ in $lc^{R_{i}}$ which is the current line segment cluster in $i$-th time interval. 
For these two outlier trajectories, two states may occur for their line segments in this time interval:
\begin{itemize}
    \item First indicates belonging to the current line segment cluster (e.g. $l_{\pi_{a}}\in lc^{R_{i}}$).
    \item Second refers to deviating from the current line segment cluster (e.g. $l_{\pi_{a}}\notin lc^{R_{i}}$).
\end{itemize}
 This process is illustrated in Fig. \ref{fig:n1}.
\begin{figure}[htbp]
    \centering
    \includegraphics[width=0.95\linewidth]{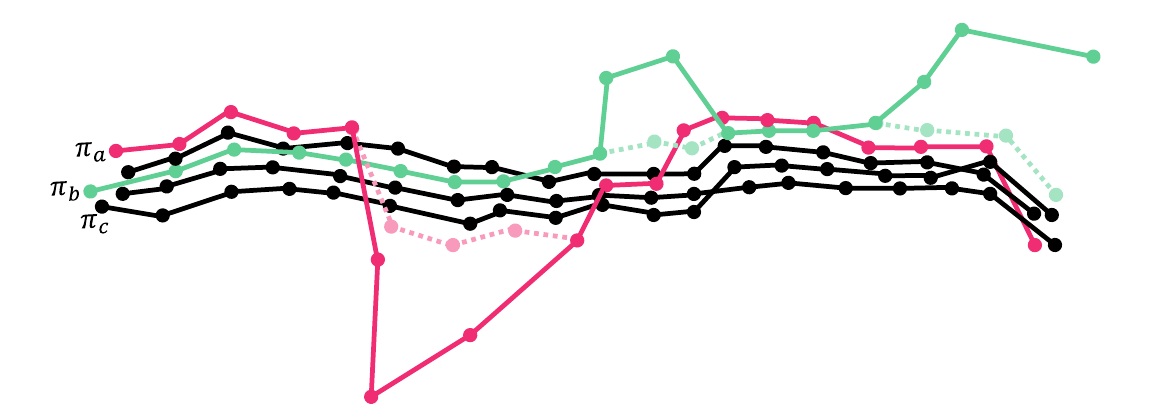}
    \caption{Two states for outlier trajectories $\pi_{a}$ and $\pi_{b}$. The trajectories in black are members of cluster $C_{num}$. The dashed line segments represent the approximate movement of objects that have deviated from their respective line segment clusters. }
    \label{fig:n1}
\end{figure}

Suppose the first state mentioned occurs within this time interval; in that case, it is sufficient to calculate the average Euclidean distance between $l_{\pi_{a}}$ and $l_{\pi_{c}}$ directly. This is because both line segments are in the same line segment cluster and are neighbors (definitions \ref{14} to \ref{91});
 otherwise, the second state occurs and the amount by which the distance between $l_{\pi_{a}}$ and $lc^{R_{i}}$ needs to be reduced for $l_{\pi_{a}}$ to become a member of $lc^{R_{i}}$ should be determined.
Consider $l_{N}$ as a borderline segment (definition \ref{90}) in $lc^{R_{i}}$. It is evident that the average Euclidean distance between $l_{\pi_{a}}$ and its nearest line segment $l_{N}$ exceeds $\varepsilon$ by $\delta_{i}$. This is because the distance being less than or equal to $\varepsilon$ implies that $l_{\pi_{a}}$ and $l_{N}$ are in the same cluster and are neighbors. Here, $\delta_{i}$ is the deviation in the $i$-th time interval. Accordingly, the average Euclidean distance between $l_{\pi_{a}}$ and $l_{\pi_{c}}$ in the current line segment cluster of the $i$-th time interval should be reduced by $\delta_{i}$. Thus, line segment $l_{\pi_{a}}$ is modified as $l'_{\pi_{a}}$. This modified line segment shows the most likely approximate movement of object $a$ in this time interval.
Therefore, the piecewise function for obtaining the set of average Euclidean distances in $T-1$ time intervals is defined as follows: 
\begin{equation}
D(\pi'_{a},\pi_{c})_{i}^{R_{i}}=
\begin{cases}
  dist(l_{\pi_{a}},l_{\pi_{c}}), & lc^{R_{i}}(l_{\pi_{a}})= lc^{R_{i}}(l_{\pi_{c}}),\\
  dist(l_{\pi_{a}},l_{\pi_{c}})-\delta_{i}, &     o.w, \\
\end{cases}
\end{equation}
where $i$ takes values in the range from $1$ to $T-1$ and \(\delta_{i} = l_{N} - \varepsilon\). This process is shown in Fig. \ref{fig:s1}.
The maximum value of these distances represents their overall average Euclidean distance and is denoted as $M_{(\pi'_{a},\pi_{c})}$. It indicates the maximum possible distance at which the outlier trajectory $\pi_{a}$ would be considered a member of cluster $C_{num}$, disregarding its anomalies.
\begin{equation}
M_{(\pi'_{a},\pi_{c})}=\max \lbrace D(\pi'_{a},\pi_{c})_{i}^{R_{i}} \quad | \quad i=1, \ldots, T-1\rbrace.
\end{equation}
The same processes occur not only for $\pi_{a}$ and the other members of cluster $C_{num}$ but also for $\pi_{b}$ and each member of cluster $C_{num}$. Eventually, the minimum of the maximum values is determined as $\mu_{min}$. It indicates which outlier trajectories $\pi_{a}$ and $\pi_{b}$, disregarding their anomalies, are more similar to and closer to the members of cluster $C_{num}$. Given that $N_{C}$ and $N_{O}$ represent the number of cluster members and outlier trajectories, respectively, $\mu_{min}$ is defined as follows:
\begin{equation}
    \mu_{min} = \min \left\{ \displaystyle M_{(\pi'_{i},\pi_{j})} \;\middle|\; 
    i=1, \ldots, N_{O}, \; j=1, \ldots, N_{c} \right\}.
    \label{eq:mumin}
\end{equation}

\begin{figure}[h!]
    \centering
    \includegraphics[width=.95\linewidth]{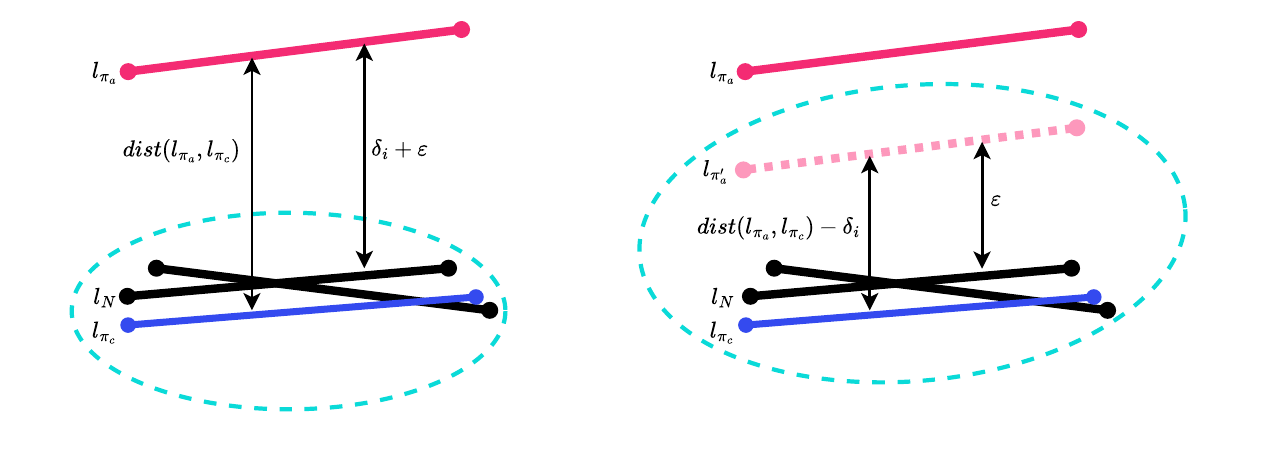}
    \caption{The average Euclidean distance between $l_{\pi_{a}}$ and $l_{\pi_{c}}$ with the assistance of the nearest line segment $l_{N}$ in current line segment cluster. $l_{\pi_{a}^{\prime}}$ is denoted as the approximate movement of the object in the $i$-th time interval.}
    \label{fig:s1}
\end{figure}

As mentioned in definition \ref{9}, the stable trajectory clustering algorithm aims to assess whether it is possible to ignore the anomalies that caused an outlier trajectory split from its respective cluster or not. 
The number of time intervals and the deviation amount in these intervals are important factors for deciding whether or not to ignore anomalies.
Specifically, one of the following four cases may occur:
\newcounter{desccount}
\newcommand{\descitem}[1]{%
  \item[#1] \refstepcounter{desccount}\label{#1}
}
\newcommand{\descref}[1]{\hyperref[#1]{#1}}
\begin{description}
    \descitem{Case 1.} The outlier trajectory deviates from its cluster in a few time intervals, but the deviation in these time intervals is large.
    \descitem{Case 2.} The outlier trajectory deviates from its cluster in many time intervals, and the deviation in these time intervals is also large.
    \descitem{Case 3.} The outlier trajectory deviates from its cluster in many time intervals, but the deviation in these time intervals is small.
    \descitem{Case 4.} The outlier trajectory deviates from its cluster in a few time intervals, and the deviation in these time intervals is also small.
\end{description}
Given these cases, we can use the mean deviation concept, which is defined as follows:
 \begin{definition}[Mean absolute deviation]\label{6}
    The mean absolute deviation is a statistical measure to determine how far the data are situated from the mean. It measures the average absolute difference between each data point and the mean. The formula \eqref{5} depicts the mean absolute deviation (MAD) for a set of $n$ data points such as $X=\left\lbrace x_{1},x_{2},...,x_{n}\right\rbrace$, where $x_{i}\in {\rm I\!R}$ and $\mu$ represents their mean:
    \begin{equation}\label{5}
        MAD=\frac{1}{n}\sum_{i=1}^{n}\abs{\mu-x_i}.
    \end{equation}
     In this formula, the deviations to the left and right of the mean are important.
     Thus, the sum of absolute differences between each right deviation (RD), which refers to the deviation of values that are greater than the mean, is denoted as the right mean deviation (RMD), while the sum of absolute differences between each left deviation (LD), which refers to the deviation of values that are less than the mean, is denoted as the left mean deviation (LMD). If LMD and RMD are not equal, then the data distribution is asymmetric or skewed.       
\end{definition}

Based on the aforementioned concepts and $\mu_{min}$ (defined in Eq. \ref{eq:mumin}), the goal is to determine which two outlier trajectories can belong to the cluster $C_{num}$, despite anomalies in a few time intervals. To put it differently, it is sufficient only to determine which members of cluster $C_{num}$ can be the neighbors of $\pi_{a}$ and $\pi_{b}$. To this aim, for each outlier trajectory like $\pi_{a}$ and each member of cluster $C_{num}$, the real average Euclidean distances of their line segments that are greater than $\mu_{min}$ are considered as RD and the real average Euclidean distances that are less than $\mu_{min}$ are regarded as LD. Thus, the sum of absolute differences between each right deviation and $\mu_{min}$ is denoted as RMD, while the sum of absolute differences between each left deviation and $\mu_{min}$ is denoted as LMD. These definitions are as follows:
\begin{equation}
\begin{aligned}
 RMD=\sum_{RD>\mu_{min}}\abs{RD-\mu_{min}},\\
 LMD=\sum_{LD<\mu_{min}}\abs{LD-\mu_{min}}.    
\end{aligned}
\end{equation}
Taking into account the aforementioned formula, and based on the fact that the average Euclidean distance of an outlier trajectory to cluster $C_{num}$ must be $\mu_{min}$ for it to be considered a member, if $RMD$ is greater than $LMD$, then a real split has occurred, and the deviations cannot be ignored. Specifically, this occurs when the deviation in the time intervals has been very significant (\descref{Case 1.} or \descref{Case 2.}) or the number of time intervals in which the trajectory deviates from its cluster is more than the number of time intervals in which the trajectory is in its cluster (\descref{Case 3.}); in contrast, if RMD is less than LMD, then, despite the anomalies, they belong to the same cluster and are stable in cluster $C_{num}$ (\descref{Case 4.}).

The stable trajectory clustering is presented in Algorithm \ref{alg:6}.
This method improves trajectory clustering results by evaluating outlier trajectories and their association with existing clusters.
It begins by checking if there are any outlier trajectories. If outliers exist, the algorithm iterates through each cluster, comparing each trajectory's line segments with those of other trajectories, and calculating distances and deviations.
It identifies the maximum distance among all these deviations and sets a minimum threshold, $\mu_{min}$. The algorithm further refines the cluster assignments by computing real distances and evaluating right and left deviations, deciding which outlier trajectories can remain in the respective cluster despite their deviations. If any cluster assignments are altered, it updates the clusters and returns the new set; otherwise, it returns the original set. 

\begin{algorithm}[!h]
\newcommand\mycommfont[1]{\small\ttfamily\textcolor{black}{#1}}
\SetCommentSty{mycommfont}
\SetKwData{Left}{left}
\SetKwData{This}{this}
\SetKwData{Up}{up}
\SetKwFunction{Union}{Union}
\SetKwFunction{FindCompress}{FindCompress}
\SetKwInOut{Input}{Input}
\SetKwInOut{Output}{Output}
\caption{Stable trajectory clustering}
\label{alg:6}
\Input{Cluster set $C$, Outlier trajectory set $\mathcal{O}$, $\varepsilon$.}
\Output{New cluster set $C$.}
\If{$\mathcal{O} = \emptyset$}{\Return $C$}
\ForEach{cluster $c_k \in C$}{
$\mathcal{M}_k \gets c_k.$allMembers;\\
    $\mathcal{O}_k \gets$ assignOutliers$(c_k, \mathcal{O})$;\\
    $\mu_{\min} \gets \infty$;\\
        \ForEach{$\pi_m \in \mathcal{M}_k, \pi_o \in \mathcal{O}_k$}{
         $M_{(\pi_m,\pi_o)} \gets -\infty$; \\
        \For{$t$ from 1 \KwTo $T-1$}{

            \eIf{$lc^{R_{t}}(l_{\pi_m}) = lc^{{R_{t}}}(l_{\pi_o})$}{                $D(\pi_{m},\pi_{o})_{t}^{R_{t}} \gets \text{dist}(l_{\pi_m}, l_{\pi_o})$;   
            }{
             $\delta \gets \text{nearestDistance}(\pi_o^t, \mathcal{M}_k^t)-\varepsilon$; \\
                $D(\pi_{m},\pi_{o})_{t}^{R_{t}} \gets \text{dist}(l_{\pi_m}, l_{\pi_o}) - \delta$;
            }
        $M_{(\pi_m,\pi_o)} \gets \max(M_{(\pi_m,\pi_o)}, D(\pi_{m},\pi_{o})_{t}^{R_{t}})$; 
        }
        $\mu_{\min} \gets \min(\mu_{\min},M_{(\pi_m,\pi_o)})$; 
    }   
}
\ForEach{$\pi_o \in \mathcal{O}_k$}{
        \ForEach{$\pi_m \in \mathcal{M}_k$}{
        $\text{LMD} \gets 0$, $\text{RMD} \gets 0$; \\
            \For{$t$ from 1 \KwTo $T-1$}{
            $D(\pi_{m},\pi_{o})_{t}^{R_{t}} \gets \text{dist}(l_{\pi_m}, l_{\pi_o})$;\\
\eIf{$D(\pi_{m},\pi_{o})_{t}^{R_{t}} < \mu_{\min}$}{$\text{LMD} \gets \text{LMD} + (\mu_{\min} - D(\pi_{m},\pi_{o})_{t}^{R_{t}})$;}{$\text{RMD} \gets \text{RMD} + (D(\pi_{m},\pi_{o})_{t}^{R_{t}} - \mu_{\min})$;}
}
        \If{$\text{LMD} > \text{RMD}$}{$c_k \gets c_k \cup \{\pi_o\}$;\\ \text{break};}
        }
   
    }
\Return $\text{updated} \; C$;
\end{algorithm}

The details of this algorithm are as follows: Initially, a function is defined to identify clusters and respective outlier trajectories (lines 3 and 4). Afterwards, an iteration is performed for each cluster and its outlier trajectories. This iteration tries to use the mean absolute deviation concept to find the outlier trajectories that, despite their anomalies, still belong to their respective clusters (lines 7- 28). Finally, if the clusters change, they will be updated and returned as a new set (line 29). The computational cost is $\mathcal{O}(|O_k| \cdot |c_k| \cdot m)$, as the algorithm evaluates the distance between every outlier and the cluster core members across all time intervals.
\par
The stable trajectory clustering algorithm can be utilized to analyze outlier trajectories, which due to their anomalies, are split from their respective cluster in some time intervals. In this analysis, it is important to consider both the number of time intervals in which deviations occur and the magnitudes of those deviations. We therefore used the concept of mean absolute deviation to decide if, despite the anomalies, outlier trajectories could be assigned to their nearest cluster. The stable trajectory clustering algorithm can be applied to both whole-trajectory clustering and sub-trajectory clustering problems. 

 \section{Empirical Results}
In this section, we empirically validate our proposed algorithms, which are the whole-trajectory clustering, the sub-trajectory clustering, and the stable trajectory clustering algorithms, using real-world trajectory data and evaluating their runtime performance. 

 \subsection{DataSets and Environment Setup}
We use two kinds of real trajectory datasets. Dataset 1 is related to a four-way traffic intersection, including maneuvers such as turns and U-turns, and contains 1900 trajectories, organized into 19 ground-truth clusters~\cite{57}. We selected this dataset because the ground truth is necessary to evaluate the impact of the stable trajectory clustering algorithm against outlier trajectories.
Dataset 2, which is known as the Vehicle Energy Dataset (VED), includes information from 383 privately-owned cars in Ann Arbor, Michigan, USA ~\cite{44}. This Dataset contains the trajectories of cars driving on both highways and traffic-dense areas from November 2017 to November 2018 and was collected by onboard OBD-II loggers.  Due to the difficulty of demonstrating the impact of the stable trajectory clustering algorithm on dense datasets, we selected this dataset to allow for a clear visualization of cluster quality.

\phantomsection
\label{para:data}
All algorithms were implemented in MATLAB 2020 and executed on a laptop with an Intel Core(TM) i7-4510U processor and 12GB of memory. Before running the algorithms, the data were preprocessed by removing duplicated spatial coordinates and adjusting all trajectories to a uniform length by linear interpolation, which was fundamental for establishing a common temporal framework for our analysis.

\subsection{Effectiveness and Parameter Sensitivity}
First, we will consider ranges of values for the two user-specified DBSCAN parameters, $\varepsilon$, and $MinLns$. For each set of DBSCAN parameters, we then apply the whole-trajectory clustering algorithms on Dataset 1. In sub-trajectory clustering, we further explore the effects of the sliding window parameters $S$ and $W$ on Dataset 1. Following that, we execute the stable trajectory clustering algorithm and analyze its effectiveness in Dataset 2.
\par
Sensitivity analysis for DBSCAN line segment clustering involves evaluating the number of clusters and cluster density across the ranges of user-specified input parameters.
These two parameters can be tuned depending on the characteristics of the data, and as mentioned before, their values impact the clustering result.
Specifically, larger values of $\varepsilon$ result in larger neighborhoods, which tend to produce larger and less compact clusters. In contrast, the value of $MinLns$ plays a significant role in assigning core or border status to line segments. Higher values require more neighboring line segments to classify a segment as core, potentially reducing the number of clusters or affecting cluster boundaries.
Fig. \ref{fig:1001} aims to provide a general overview of the impact of these two parameters on whole-trajectory clustering for Dataset 1. As expected, the number of clusters changes as a function of the values of $\varepsilon$ and $MinLns$. A similar analysis can be conducted to examine the influence of other values for these parameters, as illustrated in Fig. \ref{fig:1002}.
\begin{figure}[htbp]
    \centering
    \includegraphics[width=.95\textwidth]{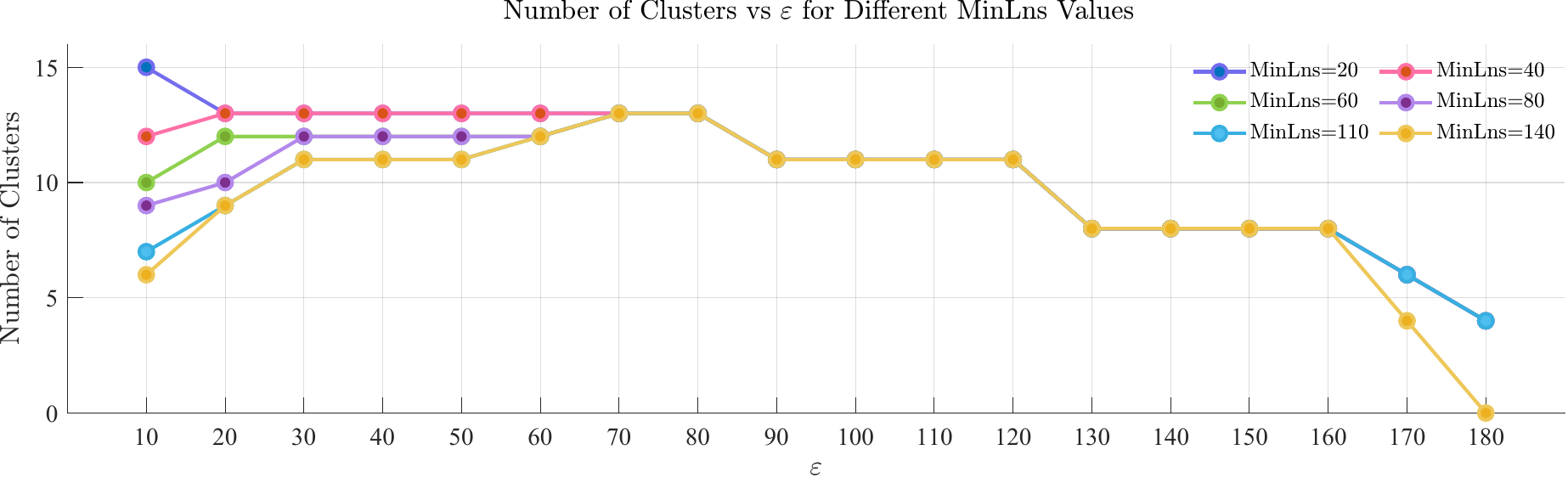}
    \caption{Effect of $\varepsilon$ and $MinLns$ on the number of clusters.}
    \label{fig:1001}
\end{figure}
\begin{figure}[ht]
    \centering
    \includegraphics[width=.95\textwidth]{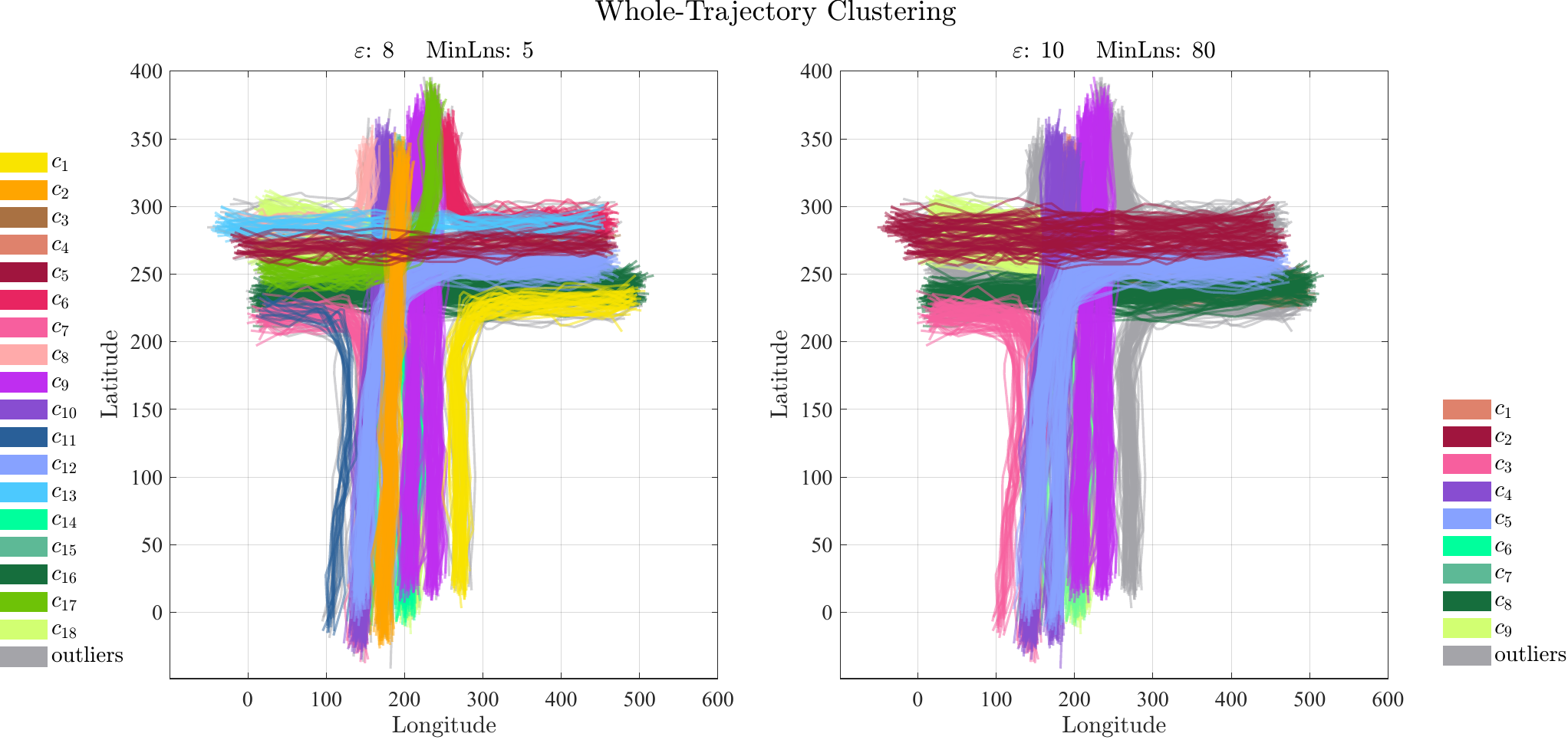}
    \caption{Whole-trajectory clustering on Dataset 1. Increasing $\varepsilon$ (from 8 to 10) and $MinLns$ (from 5 to 80) leads to denser clusters, alters their count, and impacts the number of outlier trajectories. These changes stem from line segment clustering in each time interval, impacting the overall whole-trajectory clustering.}
    \label{fig:1002}
\end{figure}

In sub-trajectory clustering, both the DBSCAN parameters and the sliding window parameters affect the results. The effects of the values $\varepsilon$ and $MinLns$ are the same as those mentioned for whole-trajectory clustering. Thus, we only present the results of varying the sliding window parameters $S$ and $W$.
\begin{figure}[htbp]
    \centering
    \begin{subfigure}[b]{1\linewidth}
        \centering
        \includegraphics[width=.95\textwidth]{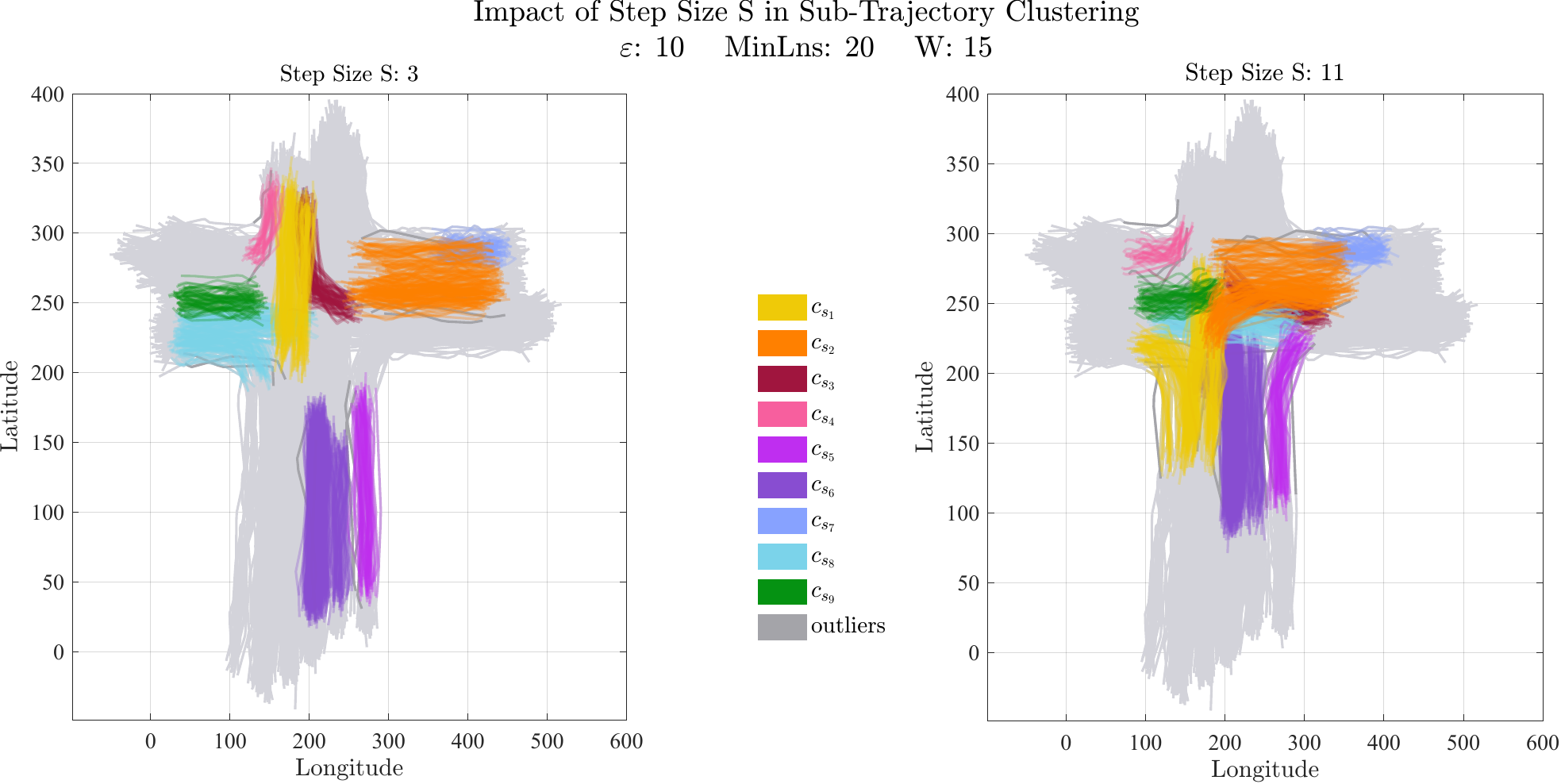}
        \caption{}
        \label{fig:1003}
    \end{subfigure}
    \hfill
    \begin{subfigure}[b]{1\linewidth}
        \centering
        \includegraphics[width=.95\textwidth]{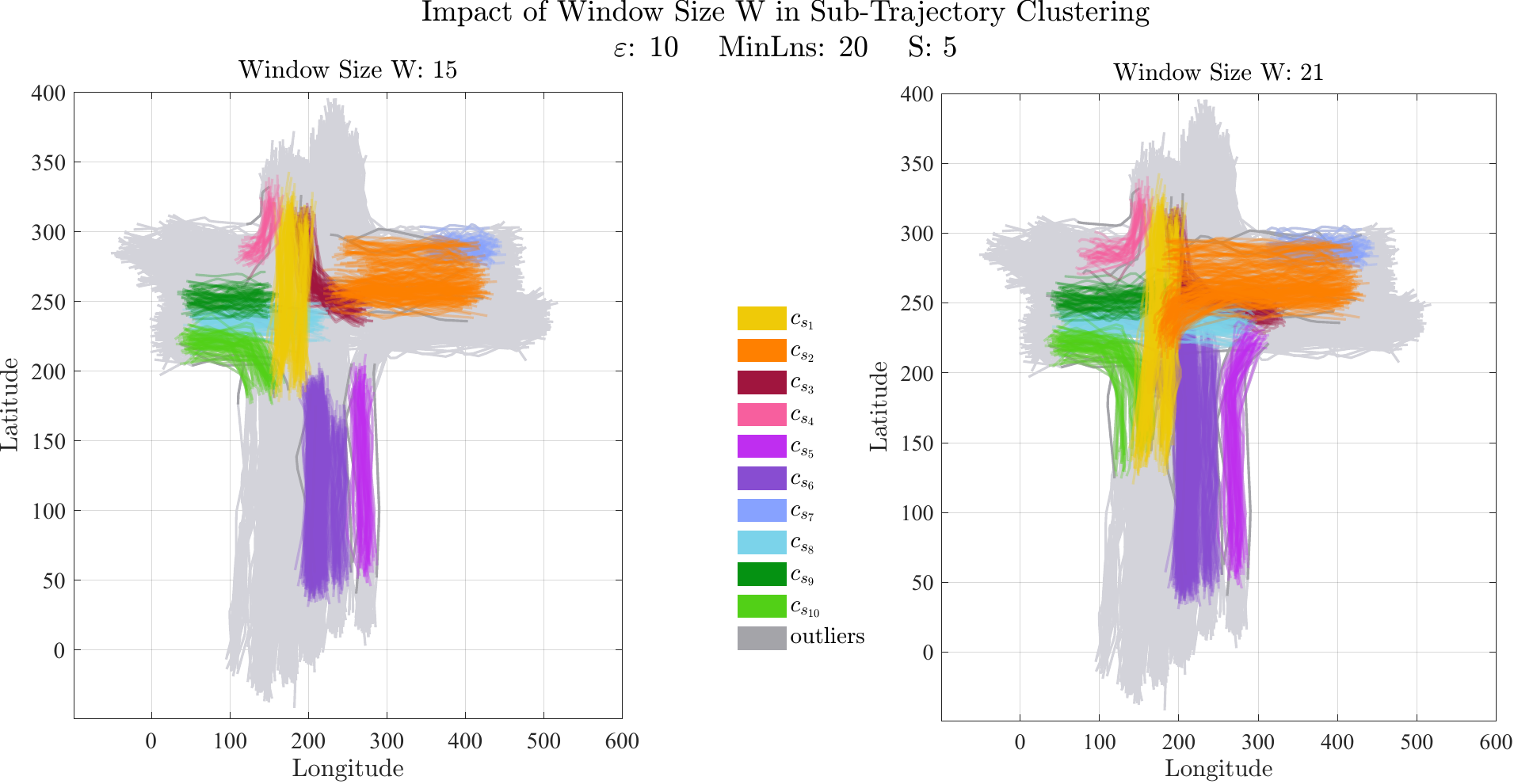}
        \caption{}
        \label{fig:1004}
    \end{subfigure}
    \caption{Sub-trajectory clustering on Dataset 1 during the second iteration of the algorithm. (a) Comparison of sub-trajectory clustering results using different step sizes ($S = 3$ and $S = 11$) with a fixed window size ($W = 15$). (b) Comparison of sub-trajectory clustering results using different window sizes ($W = 15$ and $W = 21$) with a fixed step size ($S = 5$). }
    \label{fig:mainfig}
\end{figure}
Fig. \ref{fig:mainfig} demonstrates the effects of varying $S$ and $W$ on sub-trajectory clustering for Dataset 1,  which contains trajectories divided into 50-time intervals. The results correspond to the second iteration of the algorithm.
For a fixed window size $W$, increasing the parameter $S$ causes the sliding window to move forward more significantly during the same iterations of the algorithm (\ref{fig:1003}). On the other hand, the larger the sliding window size, the more extended time intervals it encompasses, leading to a wider range of sub-trajectories (\ref{fig:1004}). 

The execution of the whole-trajectory clustering algorithm on Dataset 2 is illustrated in Fig. \ref{fig:mainfig2}. For further clarity, magnified views within the figures are provided, highlighting parts of cluster $C_{6}$ and $C_{9}$, along with their outlier trajectories in detail.
After applying the stable trajectory clustering algorithm, the minor deviations of the outlier trajectory are disregarded, and the trajectories are reassigned to their cluster (\ref{fig:1006}).
\begin{figure}[htbp]
    \centering
    \begin{subfigure}[b]{1\linewidth}
        \centering
        \includegraphics[width=.9\textwidth]{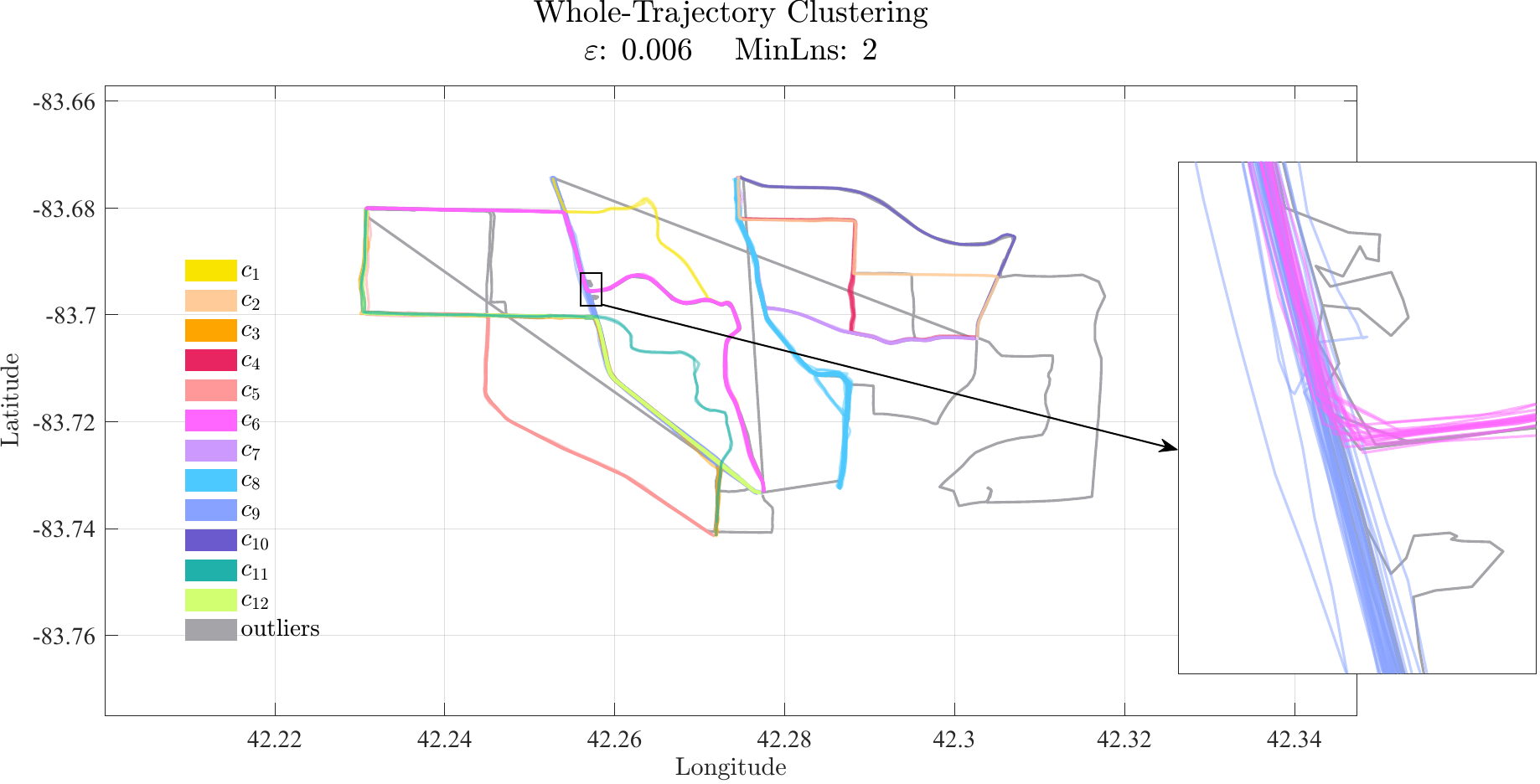}
        \caption{}
        \label{fig:1005}
    \end{subfigure}
    \hfill
    \begin{subfigure}[b]{1\linewidth}
        \centering
    \includegraphics[width=.9\textwidth]{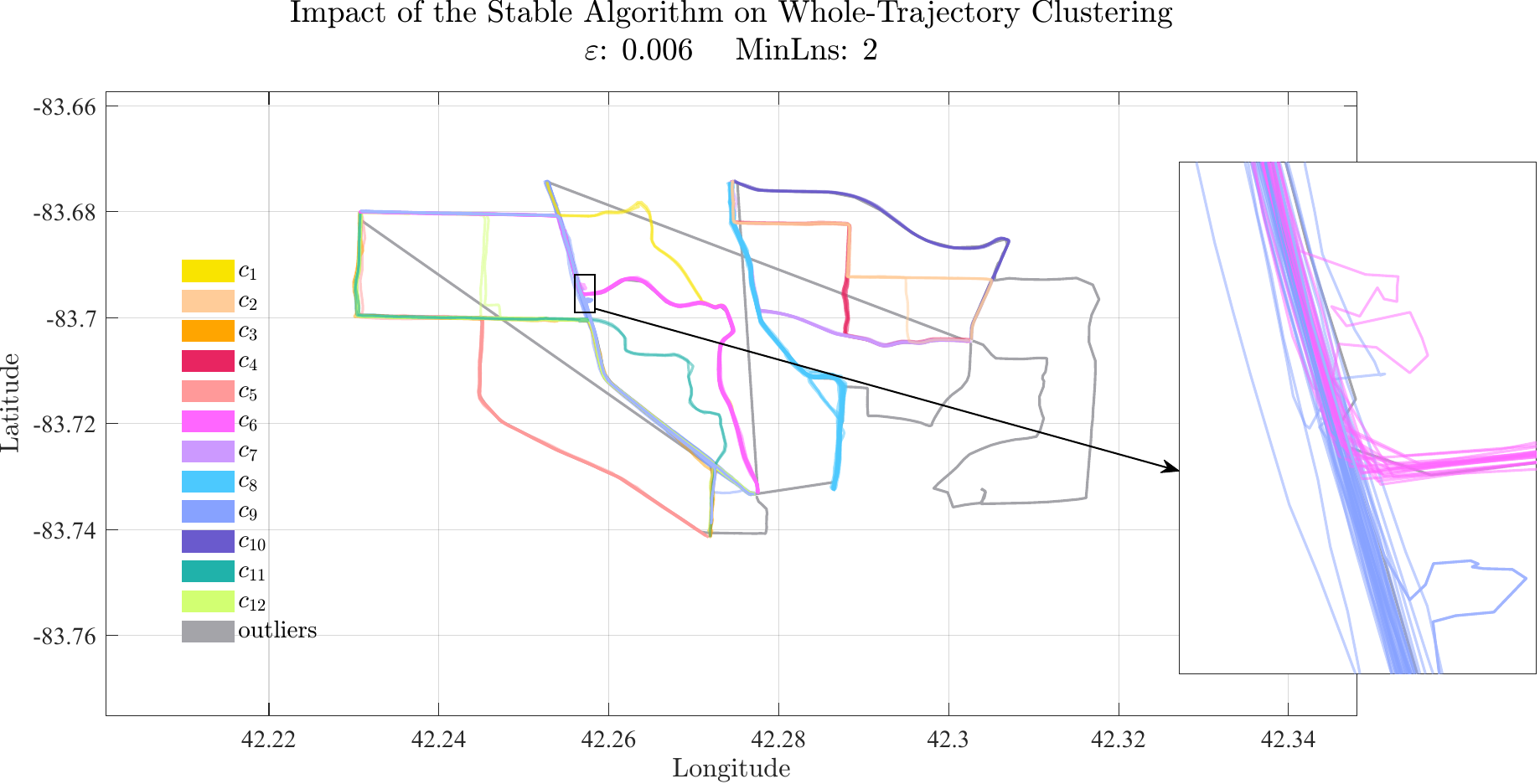}
        \caption{}
        \label{fig:1006}
    \end{subfigure}
    \caption{The stable trajectory clustering algorithm results on Dataset 2 over 101 time intervals. (a) The whole-trajectory clustering result before running the stable trajectory clustering algorithm ($\varepsilon
    =0.006$ and $MinLns=2$). The magnified view of $C_{6}$ and $C_{9}$ includes outlier trajectories that temporarily split from their cluster for a few seconds. (b) The whole trajectory clustering result after running the stable trajectory clustering algorithm. The magnified view shows that former outlier trajectories are added to $C_{6}$ and $C_{9}$.
}
    \label{fig:mainfig2}
\end{figure}
As discussed in section \ref{sec:stb}, in the stable trajectory clustering algorithm, the number of time intervals and the deviation amount in these intervals are important factors in deciding whether or not to ignore anomalies. Based on these two factors, $\mu_{min}$ is calculated for each of 12 clusters and their corresponding outlier trajectories. For example, this value for cluster $C_{6}$ and its 9 outlier trajectories is $0.0074$. Out of these, 4 are correctly assigned to their respective clusters after applying the stable clustering algorithm.
Fig. \ref{fig:1007} illustrates the temporal distance of outlier trajectories to their nearest cluster members, recognized as the distance from cluster $C_{6}$.
Each box plot represents the distribution of the distances between the segments of an outlier trajectory and the segments of its closest trajectory within the cluster across 101 time intervals. As is clear, outlier trajectories with a mean less than $\mu$ are identified as non-outlier trajectories. This means that the stable trajectory clustering algorithm can ignore insignificant anomalies in a few time intervals. These time intervals are shown in Fig. \ref{fig:1008}. 
This figure depicts each outlier trajectory as a series of 101 small rectangles, each corresponding to a time interval. It illustrates not only the time intervals during which each outlier trajectory remained within cluster $C_{6}$ but also those during which it was temporarily split from the cluster. 
\begin{figure}[htbp]
    \centering
    \includegraphics[width=.8\textwidth]{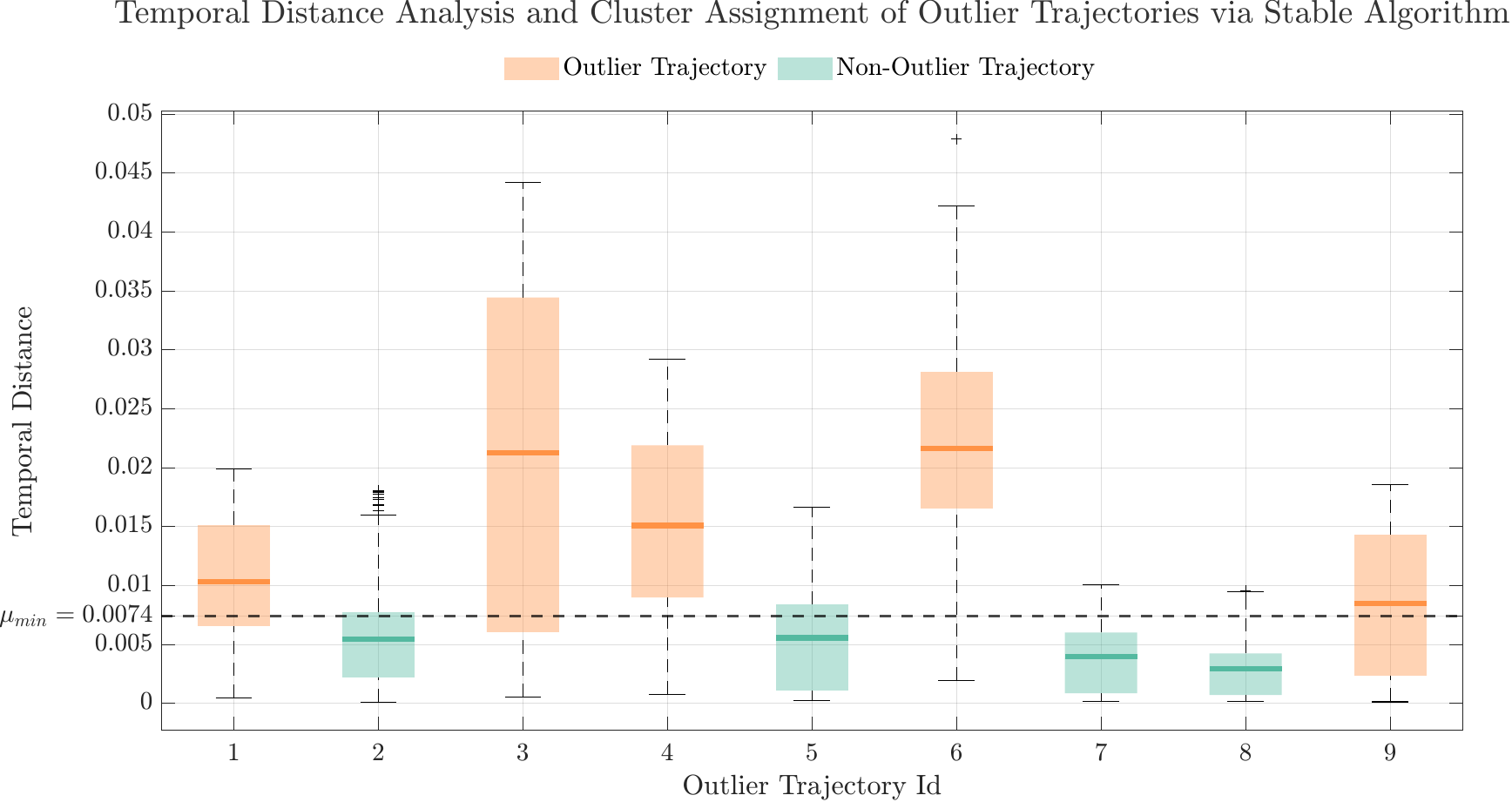}
    \caption{Temporal distance of outlier trajectories from cluster $C_{6}$. In these box plots, each line represents the average distances over 101 time intervals. The dashed horizontal line indicates $\mu_{min}$, a threshold value calculated to indicate the maximum acceptable average distance for identifying non-outlier trajectories.}
  \label{fig:1007}
\end{figure}

\begin{figure*}[htbp]
    \centering
    \includegraphics[width=1\textwidth]{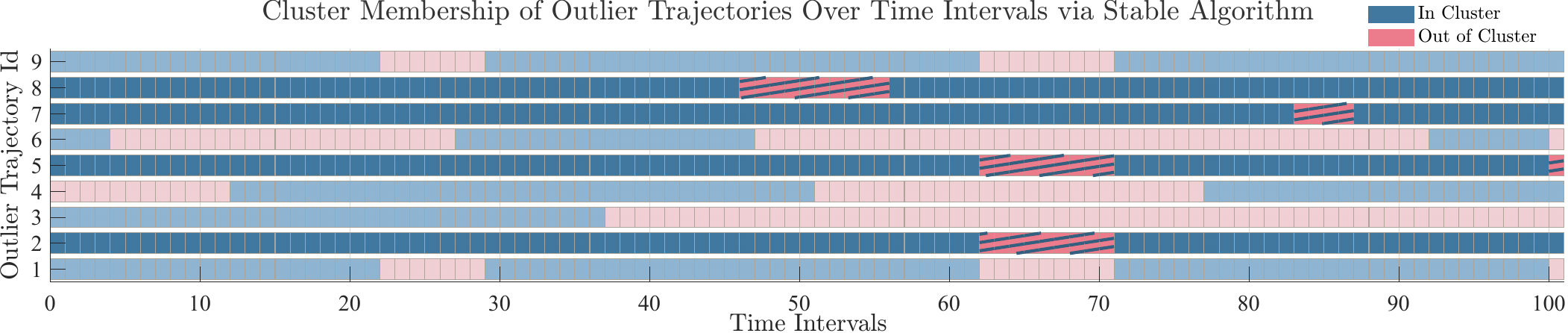}
    \caption{Temporal membership of outlier trajectories in cluster $C_{6}$ over 101-time intervals. Blue rectangles represent the time intervals during which the outlier trajectory remains within the cluster, while pink rectangles signify the time intervals when it is outside the cluster. Non-outlier trajectories are represented by the darker blue and darker pink areas, based on the indicated non-outlier trajectories in Fig. \ref{fig:1007}. The darker pink regions are also hatched, indicating time intervals where outlier trajectory deviations were disregarded.
}
    \label{fig:1008}
\end{figure*}
As mentioned before, VED (Dataset 2) provides a rich and diverse collection of real-world driving data, including GPS trajectories and energy consumption records, offering an invaluable resource for analyzing driving behavior and energy use.
Although the stable clustering algorithm on the VED data trajectories can detect minor deviations in outlier trajectories, it considers those with sharp changes in direction, which are highly noticeable, as real outlier trajectories.
This process improves the quality and reliability of the clustering results, as it considers driving patterns that reflect typical behavior. For example, a trajectory that was mistakenly labeled as an outlier due to a brief stop or a slight turn can now be correctly included in its cluster, improving the representation of real driving behavior. The improved clusters lead to more accurate energy consumption analyses and make the analysis more reliable for real-world applications like eco-driving strategies, trajectory optimization, and traffic flow studies.
This filtering slightly reduces the average Silhouette coefficient score. This metric quantifies how well data trajectories fit within their clusters (cohesion) and how distinctly they are separated from others (separation).
This reduction is due to assigning some outlier trajectories to their respective clusters after applying the stable trajectory clustering algorithm.
The Silhouette score is calculated for Dataset 1 and Dataset 2 using different DBSCAN parameter settings before and after executing the stable trajectory clustering algorithm, as shown in Table \ref{tab:clustering_results}.
In this table, in addition to this metric, the standard deviation is also computed to indicate the dispersion of the Silhouette scores of trajectories and the range within which these values vary.
The key aspect for comparison is the difference in the Silhouette score before and after executing the algorithm. The smaller this difference, and the more trajectories successfully assigned to their respective clusters, the more effective the algorithm is. A lower difference indicates that the algorithm has successfully identified outlier trajectories with minor deviations, thus preserving the overall clustering structure and maintaining the average Silhouette score with minimal change.
It should be noted that an outlier trajectory may either connect at least two clusters or be assigned to a single cluster by the stable trajectory clustering algorithm. This occurs because different segments of the outlier trajectory may belong to different clusters. After running the algorithm, the Silhouette score is evaluated for both cases, and the outlier trajectory is ultimately placed where it results in the smallest decrease in the Silhouette score.
\begin{table}[ht]
\centering
{%
\begin{tabular}{@{}ccccccccc@{}}
\toprule
\textbf{Dataset} & \multicolumn{2}{c}{\textbf{DBSCAN}} & \multicolumn{3}{c}{\textbf{Before STC}} & \multicolumn{3}{c}{\textbf{After STC}} \\ \midrule
 & \textbf{$\varepsilon$} & \textbf{MinLns} & \textbf{\#cl} & \textbf{Out.} & \textbf{Silhouette} & \textbf{\#cl} & \textbf{Out.} & \textbf{Silhouette} \\ \cmidrule(lr){2-3} \cmidrule(lr){4-6} \cmidrule(l){7-9} 
\multirow{3}{*}{\textbf{\#1}} & $2$ & $0.0075$ & $13$ & $13$ & $0.640 \pm 0.19$ & $13$ & $6$ & $0.607 \pm 0.20$ \\
 & $2$ & $0.0060$ & $12$ & $18$ & $0.653 \pm 0.21$ & $12$ & $7$ & $0.617 \pm 0.22$ \\
 & $3$ & $0.0065$ & $11$ & $17$ & $0.661 \pm 0.19$ & $11$ & $9$ & $0.623 \pm 0.20$ \\ \midrule
\multirow{3}{*}{\textbf{\#2}} & $8$ & $5$ & $18$ & $121$ & $0.709 \pm 0.19$ & $18$ & $6$ & $0.702 \pm 0.19$ \\
 & $10$ & $20$ & $15$ & $42$ & $0.750 \pm 0.12$ & $15$ & $0$ & $0.747 \pm 0.12$ \\
 & $60$ & $80$ & $11$ & $40$ & $0.773 \pm 0.08$ & $11$ & $0$ & $0.722 \pm 0.17$ \\ \bottomrule
\end{tabular}
}
 \caption{Comparison of clustering results before and after applying the stable trajectory clustering (STC) algorithm for two datasets. The results show how STC improves the cluster structure as it assigns outlier trajectories (Out.) to their respective clusters, which causes a slight reduction in the Silhouette score. A smaller difference in the Silhouette score before and after STC indicates higher effectiveness.
}
    \label{tab:clustering_results}
\end{table}
Table \ref{tab:runtime_clean} demonstrates the runtime performance of trajectory clustering approaches, where simpler datasets and lower parameter settings yield faster results, while more complex datasets and higher parameter settings require significantly more computation time.
\begin{table}[htbp]
\centering
{%
\begin{tabular}{@{}cccccccllcccll@{}}
\toprule
\multirow{2}{*}{\textbf{Dataset}} & \multicolumn{3}{c}{\multirow{2}{*}{\textbf{Split-Merge}}}   & \multicolumn{2}{c}{\textbf{Whole-Traj.}} & \multicolumn{8}{c}{\textbf{Sub-Traj.}}                                                                  \\ \cmidrule(l){5-14} 
                                  & \multicolumn{3}{c}{}                                        & \textbf{Cl. Algo.}   & \textbf{STC}      & \multicolumn{3}{c}{\textbf{Cl. Algo.}} & \multicolumn{2}{c}{}    & \multicolumn{3}{c}{\textbf{STC}}     \\ \midrule
                                  & \textbf{$\varepsilon$} & \textbf{MinLns} & \textbf{Time} & \textbf{Time}     & \textbf{Time}  & \multicolumn{3}{c}{\textbf{Time}}   & \textbf{S} & \textbf{W} & \multicolumn{3}{c}{\textbf{Time}} \\ \cmidrule(lr){2-4} \cmidrule(lr){5-6} \cmidrule(l){7-14}  
\multirow{3}{*}{\textbf{\#1}}     & $10$                   & $20$            & $353$         & $32$              & $737$          & \multicolumn{3}{c}{$172$}           & $3$        & $15$       & \multicolumn{3}{c}{$1894$}        \\
                                  & $20$                   & $60$            & $389$         & $35$              & $957$          & \multicolumn{3}{c}{$126$}           & $5$        & $20$       & \multicolumn{3}{c}{$1423$}        \\
                                  & $40$                   & $80$            & $321$         & $31$              & $845$          & \multicolumn{3}{c}{$158$}           & $4$        & $10$       & \multicolumn{3}{c}{$1787$}        \\ \midrule
\multirow{3}{*}{\textbf{\#2}}     & $0.0060$               & $2$             & $15$          & $1$               & $15$           & \multicolumn{3}{c}{$15$}            & $1$        & $10$       & \multicolumn{3}{c}{$69$}          \\
                                  & $0.0028$               & $3$             & $13$          & $0.9$               & $21$           & \multicolumn{3}{c}{$19$}            & $1$        & $20$       & \multicolumn{3}{c}{$211$}         \\
                                  & $0.0040$               & $4$             & $14$          & $1$               & $22$           & \multicolumn{3}{c}{$25$}            & $2$        & $20$       & \multicolumn{3}{c}{$130$}         \\ \bottomrule

\end{tabular}
}
\caption{Runtime performance comparison of trajectory clustering algorithms: Split-Merge (with $\varepsilon$, MinLns), Whole-Trajectory, and Sub-Trajectory methods with parameters $S$ and $W$. Times reported in seconds for Datasets \textbf{1} and \textbf{2}.}
\label{tab:runtime_clean}
\end{table}
\subsection{Comparative Evaluation and Discussion}
We evaluate our proposed clustering method against representative baseline methods, including EDR (Edit Distance on Real sequences), LCSS (Longest common subsequence), and DWT (Dynamic Time Warping), and the T2VEC deep learning model, selected from the benchmark study by Wang et al.~\cite{61} on Dataset 1. Additionally, we investigate the impact of the STC algorithm on all considered methods. Given that the dataset required preprocessing to fit our framework, we re-implemented the baselines to ensure a fair and consistent comparison.

Table \ref{table3} presents the runtime and performance results based on NMI and ARI metrics, where higher values indicate the efficiency of the STC algorithm in disregarding insignificant anomalies, directly resulting in better clustering performance.
Based on the 19 ground-truth clusters established in the dataset, we optimized the DBSCAN hyperparameters ($\varepsilon $ and $MinLns$) to obtain 19 whole-trajectory clusters with the minimum possible number of outliers across all methods. The results show that the STC algorithm effectively reassigned outlier trajectories with only minor deviations to their appropriate clusters. Consequently, the NMI and ARI scores increased significantly across all methods.
\par
Traditional methods rely on shape-based similarity between trajectories. Although DTW achieved the most accurate clustering among the baselines, it was computationally slower than our method. LCSS performed reasonably well but incurred the highest execution time, and EDR produced the weakest results. Thus, the similarity measure adopted in our work led to better clustering outcomes compared to the traditional approaches.

The deep learning model T2VEC achieved competitive results on the dataset. However, several limitations restrict its practicality. The model depends heavily on large training datasets (millions of trajectories) to extract meaningful trajectory representations. On small or medium-sized datasets, such as the one used in this study, the risk of underfitting and weak generalization increases. Another important limitation lies in interpretability. The model produces high-dimensional vector representations, but it is difficult to explain why two trajectories are considered similar. In contrast, the geometric principles used in our method are intuitive and directly reflect observable properties of the movement, such as curvature and overall shape. Furthermore, deep learning models are sensitive to a wide range of hyperparameters that are often complex to tune. Our approach depends on two parameters ($\varepsilon$ and $MinLns$) with clear geometric meaning, and provides consistent results across different runs. Finally, we could not apply the STC algorithm to T2VEC because of a fundamental difference in how they work. STC requires checking trajectory segments within specific time intervals. However, T2VEC converts the entire trajectory into a single vector, which hides the detailed time information necessary for this process. Although T2VEC performed competitively on the dataset, its strong results can be understood as an indirect approximation of the same geometric patterns that our method captures more simply and efficiently. 
\begin{table}[h!]
\centering
\renewcommand{\arraystretch}{1.4} 

\begin{tabular}{ccccc}
\hline
Category                       & Method                              & NMI             & ARI             & Time (s)       \\ \hline
\multirow{6}{*}{Traditional}   & EDR                                 & 0.6521          & 0.3027          & 1526           \\
                               & \textbf{EDR (STC)}                  & \textbf{0.7450} & \textbf{0.4618} & \textbf{8470} \\ \cline{2-5} 
                               & LCSS                                & 0.7285          & 0.5043          & 8764           \\
                               & \textbf{LCSS (STC)}                 & \textbf{0.8794} & \textbf{0.6311} & \textbf{16080} \\ \cline{2-5} 
                               & DTW                                 & 0.8622          & 0.7197          & 2524           \\
                               & \textbf{DTW (STC)}                  & \textbf{0.9459} & \textbf{0.8359} & \textbf{12430} \\ \hline
\multirow{2}{*}{Deep Learning} & T2VEC                               & 0.8098          & 0.7014          & 1380           \\
                               & T2VEC (STC)                         & N/A               & N/A               & N/A              \\ \hline
\multirow{2}{*}{This Work}     & \textit{Whole-Traj.}                & 0.8542          & 0.6832          & 379            \\
                               & \textit{\textbf{Whole-Traj. (STC)}} & \textbf{0.9103} & \textbf{0.7340} & \textbf{912}  \\ \hline
\end{tabular}

\caption{ A performance comparison of clustering methods before and after the STC algorithm. The STC post-processing step shows a consistent boost in NMI and ARI scores for the traditional methods and our proposed approach. The $Whole-Traj. (STC)$ method indicates a strong balance between high accuracy and computational efficiency. Note: the STC is not applicable to T2VEC due to its non-geometric vectors; its execution time was measured on a GPU.
}
\label{table3}
\end{table}

\section{Conclusion}
Trajectory clustering algorithms can analyze either whole trajectories or subsets of them.
In this paper, we propose both a whole-trajectory clustering algorithm and a sub-trajectory clustering algorithm. To discover the continuous object movements in each time interval, the split and merge procedures for line segments were implemented via DBSCAN line segment clustering. We used the average Euclidean distance to measure the distance between line segments. 
The result of line segment clustering is used not only for sliding-window sub-trajectory clustering to recognize local similarities, but also in whole-trajectory clustering to capture the similarity of object movements across all time intervals.
Furthermore, we introduced a novel algorithm, called the stable trajectory clustering algorithm, that represents an advanced leap forward in the literature on trajectory clustering. This algorithm addresses scenarios where a moving object temporarily diverges from its cluster and then rejoins. 
Ignoring these minor anomalies at large scales helps analysts perform trajectory clustering without being distracted by insignificant deviations, maintains data integrity, and improves clustering accuracy. For example, vehicles stopping at traffic lights, pedestrians stopping at shop windows, or animals briefly detouring during migration should still be considered part of the main movement cluster. 
To determine whether the anomalies of outlier trajectories are real or temporary, 
the number of time intervals in which deviations occur, and the magnitudes of those deviations are important factors.
Thus, we used the mean absolute deviation concept to analyze these factors. The empirical results on real data illustrated the effectiveness of the algorithm in detecting and ignoring insignificant anomalies. For this aim, NMI, ARI, and Silhouette score criteria were used to evaluate the quality of the clustering results before and after applying the stable trajectory clustering algorithm. The comparison showed that the stable clustering trajectory algorithm consistently increased the NMI and ARI scores. The results demonstrates its efficiency in handling insignificant anomalies to effectively improve clustering quality for both our proposed whole-trajectory method and traditional clustering methods. However, it was not applied to the deep learning method due to the non-geometric nature of its embeddings.
The slight decrease in the Silhouette score is an expected consequence of this accuracy improvement. It confirms that the stable trajectory clustering algorithm focuses on creating semantically correct clusters, rather than ones that are merely geometrically compact. 

The proposed algorithms exhibit sensitivity to the input parameters $\varepsilon$ and $MinLns$, with optimal values highly dependent on the characteristics of the moving data. This represents a limitation of our approach and warrants further investigation in future research. In addition, other ways to measure the stability of trajectory clustering can be used instead of the mean absolute deviation applied in this study. The definition of stability adopted in this work serves as the basis of our proposed algorithm. However, in other approaches, such as deep learning methods, it will be implemented differently.

\section*{Acknowledgment}
This work was partially funded by the Center of Advanced Systems Understanding (CASUS), which is financed by Germany’s Federal Ministry of Education and Research (BMBF) and by SMWK with tax funds on the basis of the budget approved by the Saxon State Parliament. 
\bibliographystyle{unsrt}
\bibliography{mybibliography}

@article{1,
  title={TAD: A trajectory clustering algorithm based on spatial-temporal density analysis},
  author={Yang, Yuqing and Cai, Jianghui and Yang, Haifeng and Zhang, Jifu and Zhao, Xujun},
  journal={Expert Systems with Applications},
  volume={139},
  pages={112846},
  year={2020},
  publisher={Elsevier}
}

@article{2,
  title={An algorithm for clustering animals by species based upon daily movement},
  author={Curry, David M},
  journal={Procedia Computer Science},
  volume={36},
  pages={629--636},
  year={2014},
  publisher={Elsevier}
}

@article{3,
  title={Trajectory Clustering for Air Traffic Categorisation},
  author={Boli{\'c}, Tatjana and Castelli, Lorenzo and De Lorenzo, Andrea and Vascotto, Fulvio},
  journal={Aerospace},
  volume={9},
  number={5},
  pages={227},
  year={2022},
  publisher={MDPI}
}

@article{4,
  title={Traffic Network Identification Using Trajectory Intersection Clustering},
  author={Gerdes, Ingrid and Temme, Annette},
  journal={Aerospace},
  volume={7},
  number={12},
  pages={175},
  year={2020},
  publisher={MDPI}
}

@article{5,
  title={Trajectory clustering of air traffic flows around airports},
  author={Olive, Xavier and Morio, J{\'e}r{\^o}me},
  journal={Aerospace Science and Technology},
  volume={84},
  pages={776--781},
  year={2019},
  publisher={Elsevier}
}

@article{6,
  title={Deriving animal movement behaviors using movement parameters extracted from location data},
  author={Teimouri, Maryam and Indahl, Ulf Geir and Sickel, Hanne and Tveite, H{\aa}vard},
  journal={ISPRS International Journal of Geo-Information},
  volume={7},
  number={2},
  pages={78},
  year={2018},
  publisher={MDPI}
}

@article{7,
  title={Exploring urban travel patterns using density-based clustering with multi-attributes from large-scaled vehicle trajectories},
  author={Tang, Jinjun and Bi, Wei and Liu, Fang and Zhang, Wenhui},
  journal={Physica A: Statistical Mechanics and its Applications},
  volume={561},
  pages={125301},
  year={2021},
  publisher={Elsevier}
}

@inproceedings{8,
  title={Vector field k-means: Clustering trajectories by fitting multiple vector fields},
  author={Ferreira, Nivan and Klosowski, James T and Scheidegger, Carlos E and Silva, Cl{\' a}udio T},
  booktitle={Computer Graphics Forum},
  volume={32},
  number={3pt2},
  pages={201--210},
  year={2013},
  organization={Wiley Online Library}
}

@inproceedings{9,
  title={Clustering moving objects},
  author={Li, Yifan and Han, Jiawei and Yang, Jiong},
  booktitle={Proceedings of the tenth ACM SIGKDD international conference on Knowledge discovery and data mining},
  pages={617--622},
  year={2004}
}

@inproceedings{10,
  title={On discovery of traveling companions from streaming trajectories},
  author={Tang, Lu-An and Zheng, Yu and Yuan, Jing and Han, Jiawei and Leung, Alice and Hung, Chih-Chieh and Peng, Wen-Chih},
  booktitle={2012 IEEE 28th International Conference on Data Engineering},
  pages={186--197},
  year={2012},
  organization={IEEE}
}

@article{11,
  title={Continuous clustering of moving objects},
  author={Jensen, Christian S and Lin, Dan and Ooi, Beng Chin},
  journal={IEEE transactions on knowledge and data engineering},
  volume={19},
  number={9},
  pages={1161--1174},
  year={2007},
  publisher={IEEE}
}

@article{12,
  title={Effective online group discovery in trajectory databases},
  author={Li, Xiaohui and Ceikute, Vaida and Jensen, Christian S and Tan, Kian-Lee},
  journal={IEEE Transactions on Knowledge and Data Engineering},
  volume={25},
  number={12},
  pages={2752--2766},
  year={2012},
  publisher={IEEE}
}

@inproceedings{13,
  title={The definition and computation of trajectory and subtrajectory similarity},
  author={Van Kreveld, Marc and Luo, Jun},
  booktitle={Proceedings of the 15th annual ACM international symposium on Advances in geographic information systems},
  pages={1--4},
  year={2007}
}

@inproceedings{14,
  title={Finding long and similar parts of trajectories},
  author={Buchin, Kevin and Buchin, Maike and Van Kreveld, Marc and Luo, Jun},
  booktitle={Proceedings of the 17th ACM SIGSPATIAL International Conference on Advances in Geographic Information Systems},
  pages={296--305},
  year={2009}
}

@article{15,
  title={Detecting commuting patterns by clustering subtrajectories},
  author={Buchin, Kevin and Buchin, Maike and Gudmundsson, Joachim and L{\"o}ffler, Maarten and Luo, Jun},
  journal={International Journal of Computational Geometry \& Applications},
  volume={21},
  number={03},
  pages={253--282},
  year={2011},
  publisher={World Scientific}
}

@inproceedings{16,
  title={Trajectory clustering: a partition-and-group framework},
  author={Lee, Jae-Gil and Han, Jiawei and Whang, Kyu-Young},
  booktitle={Proceedings of the 2007 ACM SIGMOD international conference on Management of data},
  pages={593--604},
  year={2007}
}

@article{17,
  title={Online clustering for trajectory data stream of moving objects},
  author={Yu, Yanwei and Wang, Qin and Wang, Xiaodong and Wang, Huan and He, Jie},
  journal={Computer science and information systems},
  volume={10},
  number={3},
  pages={1293--1317},
  year={2013}
}

@inproceedings{18,
  title={Stability of k-means clustering},
  author={Ben-David, Shai and P{\'a}l, D{\'a}vid and Simon, Hans Ulrich},
  booktitle={Learning Theory: 20th Annual Conference on Learning Theory, COLT 2007, San Diego, CA, USA; June 13-15, 2007. Proceedings 20},
  pages={20--34},
  year={2007},
  organization={Springer}
}

@article{19,
  title={Dealing with noise in crowdsourced GPS human trajectory logging data},
  author={Adhinugraha, Kiki and Rahayu, Wenny and Hara, Takahiro and Taniar, David},
  journal={Concurrency and Computation: Practice and Experience},
  volume={33},
  number={19},
  pages={e6139},
  year={2021},
  publisher={Wiley Online Library}
}

@article{20,
  title={Detecting anomalous trajectories and behavior patterns using hierarchical clustering from taxi GPS data},
  author={Wang, Yulong and Qin, Kun and Chen, Yixiang and Zhao, Pengxiang},
  journal={ISPRS International Journal of Geo-Information},
  volume={7},
  number={1},
  pages={25},
  year={2018},
  publisher={MDPI}
}

@article{21,
  title={Grid-Based Whole Trajectory Clustering in Road Networks Environment},
  author={Wang, Fangshu and Wang, Shuai and Niu, Xinzheng and Zhu, Jiahui and Chen, Ting},
  journal={Wireless Communications and Mobile Computing},
  volume={2021},
  pages={1--20},
  year={2021},
  publisher={Hindawi Limited}
}

@inproceedings{22,
  title={A density-based algorithm for discovering clusters in large spatial databases with noise.},
  author={Ester, Martin and Kriegel, Hans-Peter and Sander, J{\"o}rg and Xu, Xiaowei and others},
  booktitle={kdd},
  volume={96},
  number={34},
  pages={226--231},
  year={1996}
}

@book{23,
  title={Data mining: concepts and techniques},
  author={Han, Jiawei and Pei, Jian and Kamber, Micheline},
  year={2011},
  publisher={Elsevier}
}

@techreport{24,
  title={k-means++: The advantages of careful seeding},
  author={Arthur, David and Vassilvitskii, Sergei},
  year={2006},
  institution={Stanford}
}

@article{25,title={Least squares quantization in PCM},
  author={Lloyd, Stuart},
  journal={IEEE transactions on information theory},
  volume={28},
  number={2},
  pages={129--137},
  year={1982},
  publisher={IEEE}
}

@article{26,
  title={K-modes clustering},
  author={Chaturvedi, Anil and Green, Paul E and Caroll, J Douglas},
  journal={Journal of classification},
  volume={18},
  number={1},
  pages={35--55},
  year={2001},
  publisher={Springer}
}

@article{27,
  title={OPTICS: Ordering points to identify the clustering structure},
  author={Ankerst, Mihael and Breunig, Markus M and Kriegel, Hans-Peter and Sander, J{\"o}rg},
  journal={ACM Sigmod record},
  volume={28},
  number={2},
  pages={49--60},
  year={1999},
  publisher={ACM New York, NY, USA}
}

@article{28,
  title={BIRCH: an efficient data clustering method for very large databases},
  author={Zhang, Tian and Ramakrishnan, Raghu and Livny, Miron},
  journal={ACM sig
  record},
  volume={25},
  number={2},
  pages={103--114},
  year={1996},
  publisher={ACM New York, NY, USA}
}

@article{29,
  title={CURE: An efficient clustering algorithm for large databases},
  author={Guha, Sudipto and Rastogi, Rajeev and Shim, Kyuseok},
  journal={ACM Sigmod record},
  volume={27},
  number={2},
  pages={73--84},
  year={1998},
  publisher={ACM New York, NY, USA}
}

@inproceedings{30,
  title={STING: A statistical information grid approach to spatial data mining},
  author={Wang, Wei and Yang, Jiong and Muntz, Richard and others},
  booktitle={Vldb},
  volume={97},
  pages={186--195},
  year={1997},
  organization={Citeseer}
}

@inproceedings{31,
  title={Automatic subspace clustering of high dimensional data for data mining applications},
  author={Agrawal, Rakesh and Gehrke, Johannes and Gunopulos, Dimitrios and Raghavan, Prabhakar},
  booktitle={Proceedings of the 1998 ACM SIGMOD international conference on Management of data},
  pages={94--105},
  year={1998}
}

@article{32,
  title={Improved node localization using K-means clustering for Wireless Sensor Networks},
  author={El Khediri, Salim and Fakhet, Walid and Moulahi, Tarek and Khan, Rehanullah and Thaljaoui, Adel and Kachouri, Abdennaceur},
  journal={Computer Science Review},
  volume={37},
  pages={100284},
  year={2020},
  publisher={Elsevier}
}

@inproceedings{33,
  title={An application of dbscan clustering for flight anomaly detection during the approach phase},
  author={Sheridan, Kevin and Puranik, Tejas G and Mangortey, Eugene and Pinon-Fischer, Olivia J and Kirby, Michelle and Mavris, Dimitri N},
  booktitle={AIAA Scitech 2020 Forum},
  pages={1851},
  year={2020}
}

@article{34,
  title={Real-time superpixel segmentation by DBSCAN clustering algorithm},
  author={Shen, Jianbing and Hao, Xiaopeng and Liang, Zhiyuan and Liu, Yu and Wang, Wenguan and Shao, Ling},
  journal={IEEE transactions on image processing},
  volume={25},
  number={12},
  pages={5933--5942},
  year={2016},
  publisher={IEEE}
}

@article{35,
  title={Adaptive density trajectory cluster based on time and space distance},
  author={Liu, Fagui and Zhang, Zhijie},
  journal={Physica A: Statistical Mechanics and its Applications},
  volume={484},
  pages={41--56},
  year={2017},
  publisher={Elsevier}
}

@article{36,
  title={Online clustering of streaming trajectories},
  author={Mao, Jiali and Song, Qiuge and Jin, Cheqing and Zhang, Zhigang and Zhou, Aoying},
  journal={Frontiers of Computer Science},
  volume={12},
  pages={245--263},
  year={2018},
  publisher={Springer}
}

@article{37,
  title={Visual cluster analysis of trajectory data with interactive kohonen maps},
  author={Schreck, Tobias and Bernard, J{\"u}rgen and Von Landesberger, Tatiana and Kohlhammer, J{\"o}rn},
  journal={Information Visualization},
  volume={8},
  number={1},
  pages={14--29},
  year={2009},
  publisher={SAGE Publications Sage UK: London, England}
}

@inproceedings{38,
  title={Online clustering of trajectory data stream},
  author={Da Silva, Ticiana L Coelho and Zeitouni, Karine and de Mac{\^e}do, Jos{\'e} AF},
  booktitle={2016 17th IEEE International Conference on Mobile Data Management (MDM)},
  volume={1},
  pages={112--121},
  year={2016},
  organization={IEEE}
}

@article{41,
  title={Tracking clusters in evolving data streams over sliding windows},
  author={Zhou, Aoying and Cao, Feng and Qian, Weining and Jin, Cheqing},
  journal={Knowledge and Information Systems},
  volume={15},
  pages={181--214},
  year={2008},
  publisher={Springer}
}

@inproceedings{42,
  title={Challenges and issues in trajectory streams clustering upon a Sliding-Window Model},
  author={Mao, Jiali and Jin, Cheqing and Wang, Xiaoling and Zhou, Aoying},
  booktitle={2015 12th Web Information System and Application Conference (WISA)},
  pages={303--308},
  year={2015},
  organization={IEEE}
}

@inproceedings{43,
  title={TSCluWin: Trajectory stream clustering over sliding window},
  author={Mao, Jiali and Song, Qiuge and Jin, Cheqing and Zhang, Zhigang and Zhou, Aoying},
  booktitle={Database Systems for Advanced Applications: 21st International Conference, DASFAA 2016, Dallas, TX, USA, April 16-19, 2016, Proceedings, Part II 21},
  pages={133--148},
  year={2016},
  organization={Springer}
}

@article{44,
  title={Vehicle energy dataset (VED), a large-scale dataset for vehicle energy consumption research},
  author={Oh, Geunseob and Leblanc, David J and Peng, Huei},
  journal={IEEE Transactions on Intelligent Transportation Systems},
  volume={23},
  number={4},
  pages={3302--3312},
  year={2020},
  publisher={IEEE}
}

@article{45,
  title={A systematic approach to clustering whole trajectories of mobile objects in road networks},
  author={Han, Binh and Liu, Ling and Omiecinski, Edward},
  journal={IEEE Transactions on Knowledge and Data Engineering},
  volume={29},
  number={5},
  pages={936--949},
  year={2017},
  publisher={IEEE}
}

@article{46,
  title={A visual-numeric approach to clustering and anomaly detection for trajectory data},
  author={Kumar, Dheeraj and Bezdek, James C and Rajasegarar, Sutharshan and Leckie, Christopher and Palaniswami, Marimuthu},
  journal={The Visual Computer},
  volume={33},
  pages={265--281},
  year={2017},
  publisher={Springer}
}

@inproceedings{47,
  title={A stable clustering algorithm based on affinity propagation for VANETs},
  author={Shahwani, Hamayoun and Bui, Toan Duc and Jeong, Jaehoon Paul and Shin, Jitae},
  booktitle={2017 19th International Conference on Advanced Communication Technology (ICACT)},
  pages={501--504},
  year={2017},
  organization={IEEE}
}

@article{48,
  title={An efficient and distributed framework for real-time trajectory stream clustering},
  author={Gao, Yunjun and Fang, Ziquan and Xu, Jiachen and Gong, Shenghao and Shen, Chunhui and Chen, Lu},
  journal={IEEE Transactions on Knowledge and Data Engineering},
  year={2023},
  publisher={IEEE}
}

@article{50,
  title={Continuous trajectory similarity search for online outlier detection},
  author={Zhang, Dongxiang and Chang, Zhihao and Wu, Sai and Yuan, Ye and Tan, Kian-Lee and Chen, Gang},
  journal={IEEE Transactions on Knowledge and Data Engineering},
  volume={34},
  number={10},
  pages={4690--4704},
  year={2020},
  publisher={IEEE}
}

@article{54,
  title={A stable clustering algorithm using the traffic regularity of buses in urban VANET scenarios},
  author={Tseng, Hsueh-Wen and Wu, Ruei-Yu and Lo, Ching-Wen},
  journal={Wireless Networks},
  volume={26},
  pages={2665--2679},
  year={2020},
  publisher={Springer}
}

@inproceedings{57,
  title={Learning trajectory patterns by clustering: Experimental studies and comparative evaluation},
  author={Morris, Brendan and Trivedi, Mohan},
  booktitle={2009 IEEE Conference on Computer Vision and Pattern Recognition},
  pages={312--319},
  year={2009},
  organization={IEEE}
}

@article{58,
  title={A survey of trajectory distance measures and performance evaluation},
  author={Su, Han and Liu, Shuncheng and Zheng, Bolong and Zhou, Xiaofang and Zheng, Kai},
  journal={The VLDB Journal},
  volume={29},
  number={1},
  pages={3--32},
  year={2020},
  publisher={Springer}
}

@inproceedings{59,
  title={Trajectory clustering via deep representation learning},
  author={Yao, Di and Zhang, Chao and Zhu, Zhihua and Huang, Jianhui and Bi, Jingping},
  booktitle={2017 international joint conference on neural networks (IJCNN)},
  pages={3880--3887},
  year={2017},
  organization={IEEE}
}

@article{60,
  title={A deep trajectory clustering method based on sequence-to-sequence autoencoder model},
  author={Wang, Chao and Lyu, Fangzheng and Wu, Sensen and Wang, Yuanyuan and Xu, Liuchang and Zhang, Feng and Wang, Shaowen and Wang, Yongheng and Du, Zhenhong},
  journal={Transactions in GIS},
  volume={26},
  number={4},
  pages={1801--1820},
  year={2022},
  publisher={Wiley Online Library}
}

@article{61,
  title={A deep spatiotemporal trajectory representation learning framework for clustering},
  author={Wang, Chao and Huang, Jiahui and Wang, Yongheng and Lin, Zhengxuan and Jin, Xiongnan and Jin, Xing and Weng, Di and Wu, Yingcai},
  journal={IEEE Transactions on Intelligent Transportation Systems},
  volume={25},
  number={7},
  pages={7687--7700},
  year={2024},
  publisher={IEEE}
}

@article{62,
  title={Trajectory outlier detection: New problems and solutions for smart cities},
  author={Djenouri, Youcef and Djenouri, Djamel and Lin, Jerry Chun-Wei},
  journal={ACM Transactions on Knowledge Discovery from Data (TKDD)},
  volume={15},
  number={2},
  pages={1--28},
  year={2021},
  publisher={ACM New York, NY, USA}
}

@article{63,
  title={Clustering users by their mobility behavioral patterns},
  author={Ben-Gal, Irad and Weinstock, Shahar and Singer, Gonen and Bambos, Nicholas},
  journal={ACM Transactions on Knowledge Discovery from Data (TKDD)},
  volume={13},
  number={4},
  pages={1--28},
  year={2019},
  publisher={ACM New York, NY, USA}
}
\end{document}